\documentclass{article}

\usepackage[preprint]{neurips_2021}

\usepackage[utf8]{inputenc} % allow utf-8 input
\usepackage[T1]{fontenc}    % use 8-bit T1 fonts
\usepackage{hyperref}       % hyperlinks
\usepackage{url}            % simple URL typesetting
\usepackage{booktabs}       % professional-quality tables
\usepackage{amsfonts}       % blackboard math symbols
\usepackage{nicefrac}       % compact symbols for 1/2, etc.
\usepackage{microtype}      % microtypography
\usepackage{xcolor}         % colors
\usepackage{enumerate}
\usepackage{amsmath,amsthm,amssymb}
\usepackage{algorithm,algorithmicx,algpseudocode}

\usepackage{natbib}

\newtheorem{theorem}{Theorem}
\newtheorem{lemma}{Lemma}

\newtheorem{remark}{Remark}

\newcommand{\cA}{\mathcal{A}}
\newcommand{\cB}{\mathcal{B}}

\newcommand{\cH}{\mathcal{H}}
\newcommand{\cT}{\mathcal{T}}
\newcommand{\cK}{\mathcal{K}}
\newcommand{\cE}{\mathcal{E}}

\title{Cooperative Stochastic Multi-agent Multi-armed Bandits Robust to Adversarial Corruptions}

\author{%
  Junyan~Liu \\
  University California San Diego\\
  \texttt{jul037@ucsd.edu} \\
  \And
  Shuai~Li \\
  Shanghai Jiao Tong University\\
  \texttt{shuaili8@sjtu.edu.cn} \\
  \And
  Dapeng~Li \\
  Nanjing University of Posts and Telecommunications\\
  \texttt{dapengli@njupt.edu.cn} \\
}

\begin{document}

\maketitle

\begin{abstract}
% !TEX root = main.tex

We study the problem of stochastic bandits with adversarial corruptions in the cooperative multi-agent setting, where $V$ agents interact with a common $K$-armed bandit problem, and each pair of agents can communicate with each other to expedite the learning process.
In the problem, the rewards are independently sampled from distributions across all agents and rounds, but they may be corrupted by an adversary. Our goal is to minimize both the overall regret and communication cost across all agents. We first show that an additive term of corruption is unavoidable for any algorithm in this problem. Then, we propose a new algorithm that is agnostic to the level of corruption. Our algorithm not only achieves near-optimal regret in the stochastic setting, but also obtains a regret with an additive term of corruption in the corrupted setting, while maintaining efficient communication. The algorithm is also applicable for the single-agent corruption problem, and achieves a high probability regret that removes the multiplicative dependence of $K$ on corruption level. Our result of the single-agent case resolves an open question from \cite{COLT19Gupta}.

%Our algorithm enables the agent to correct the suboptimal selection by communicating with other agents. 
\end{abstract}

\section{Introduction}
\label{sec:intro}
% !TEX root = main.tex

The multi-armed bandit (MAB) problem is one of the most fundamental problems in online learning. Motivated by the emerging need for cooperative multi-agent learning in large-scale systems, researchers sought to expedite the learning process for MAB problem by exploiting the multi-agent cooperation. Examples include recommendation systems \cite{DMABsigmetrics}, forage task of robotics \cite{Robotics}, channel allocation in wireless networks \cite{DMABTIT1,DMABTIT2}. The learning process of those applications is inherently distributed, and due to geographical dispersion, the communication between each agent often comes at a heavy price.
Therefore, apart from the exploration-exploitation trade-off in the basic MAB problem, the multi-agent MAB problem needs to tackle with an additional trade-off between the regret and communication cost.

The multi-agent MAB problem is typically studied in the (i) \emph{stochastic setting} \cite{DistributedMABICML1,NIPSDistributedPure2,
FOCSDistributedPure1,DistributedMABICLR1,
DistributedMAB4,DistributedMAB3,DMABSto1,DMABSto2}, i.e., the rewards are always sampled from some underlying distributions, or (ii) \emph{adversarial setting} \cite{DistributedMABADVCOLT1,
DistributedADVMABNIPS1,
DistributedMAB2,DMABADVNIPS1,DMABADVNIPS2}, i.e., the rewards are always corrupted/manipulated by an adversary. However, these two extremes are not appropriate for many real-world situations, e.g., click fraud \cite{STOCMultiLayer}, fake reviewers \cite{FIXEDCORRUPTION}, and denial of service (DoS) in path routing \cite{TONbestOfBothWorlds} that are better considered as partially stochastic and partially adversarial. Toward this end, it is appealing to consider an environment that sits in between two extremes and can adaptively go from the stochastic setting to the adversarial one, with the increase of corruption. Moreover, in these situations, agents are more interested in the stochastic structure, e.g., click preference of users in the click fraud scenario, rather than the observed rewards, and thus the regret is better evaluated in the stochastic bandits framework.

This paper considers the stochastic multi-agent bandits with adversarial corruptions problem which bridges the gap
between two extremes. In such a setting, the rewards initially sampled from the underlying distributions might be corrupted by an adversary before revealed to agents. This problem is recently considered for the single-agent bandit learning \cite{STOCMultiLayer,COLT19Gupta,FIXEDCORRUPTION,
JMLRTsallis,DuelingCorruption,linearCorruption2,
linearCorruption1,LHPcorruption1}, but
little is known about the cooperative multi-agent case. Different from the corruption problem in the single-agent setup, the crux of adversarial corruption in the multi-agent case is to balance the communication-regret trade-off under adversarial corruptions. Our objective of this problem is threefold: (i) achieving regret close to that incurred by an optimal centralized algorithm running for
$VT$ rounds in the stochastic setting; (ii) maintaining efficient communication; (iii) robust to the corruption.

\subsection{Problem Setting} \label{problemsetup}
Consider a set of $V$ agents, represented by $ [V]=\{1,2,...,V\}$, is given access to a same bandit instance where all agents interact with $K$ arms represented by $[K]=\{1,2,...,K\}$. Assume that $V\leq K$ like previous works \cite{DistributedMAB4,PragnyaLevy,CongShenAISTATS,MyFairBandit}. Each arm $i \in [K]$ is associated with a reward whose distribution $D_i$ and mean $\mu_i$ are unknown to agents. All rewards are assumed to be supported on $[0,1]$. Specifically, the interaction between agents and the adversary is formulated as: at each round $t = 1,...,T$,
\begin{itemize}
\item  For each agent $v \in [V]$, a stochastic reward $r^{S}_{v,i}(t) \sim D_{i}$ for each arm is generated and they together form a vector $r_v^{S}(t) \in [0,1]^K$;
\item  The adversary observes $\{r_v^{S}(t) \}_{v=1}^V$, and then returns a corrupted reward vector $r_v(t) \in [0,1]^K$ with entry $r_{v,i}(t)$ for each agent $v$;
\item  Each agent $v$ simultaneously pulls an arm $i_v(t)=i$ and observes the corrupted reward $r_{i_v(t)}(t)$.
\end{itemize}

The agents are unaware of which arms other agents pull at the current round, but may know the observations of others from previous rounds with the cost of communication. The communication is allowed for any pair of agents, and each agent may broadcast a message with roughly $\widetilde{O}(K)$\footnote{We use $\widetilde{O}(\cdot)$ notation to suppress all polylogarithmic factors.} bits to other agent(s) at any round. We assume that the communication between agents has no delay and simultaneously pulling the same arm leads to no collision.

The corruption level across all agents and rounds is measured as
\begin{equation} \label{corruptionDef}
C= \sum_{t=1}^T   \sum_{v=1}^V \left\Vert r_v(t) - r_v^{S}(t)  \right\Vert_{\infty}  .
\end{equation}

The adversary is allowed to be \textit{adaptive}, in the sense that the corruption at round $t$ is determined as a function of all previous arm selections and corresponding rewards. The corruption in any round is independent of the choice of arms in the same round across all agents. It is noteworthy that the corruption level $C$ is a \textit{random variable} that depends on
the randomization of stochastic rewards and the choices of agents.

The goal of agents is to minimize the communication costs that measure the total bits transmitted between all agents and the pseudo-regret (regret, for short unless otherwise stated) defined as
\begin{equation}\label{reg}
R_T  =   \sum_{t=1}^T \sum_{v=1}^V  \left( \mu_{i^*}-\mu_{i_v(t)}\right),
\end{equation}
where $i^*$ is the optimal arm such that $i^* =\arg\max_{i \in [K]} \mu_i$, and $\mu_{i^*}$ is its corresponding mean. For any suboptimal arm $i \neq i^*$, the gap of arm $i$ is $\Delta_i=\mu_{i^*}-\mu_{i} \in [0,1]$ and $\Delta_{\min}=\min_{i:\Delta_i>0} \Delta_i$.

%One can see that running optimal robust stochastic bandit algorithms in the full communication model, i.e., sharing all observations at each round, can achieve optimal regret, whereas this yields $O(T)$ communication cost. On the contrary, running the robust algorithms in the no communication model, i.e., agents never communicate, will yield no communication cost, but this incurs regret $V$ times larger than the one in the full communication model. Toward this end, we propose an algorithm to balance the communication-regret trade-off, while maintaining the robustness to the adversarial corruptions.

\subsection{Related Work}
\textbf{Multi-agent MAB.} The multi-agent MAB problelm is extensively studied in either stochastic setting or adversarial setting.

In the stochastic setting, one line of research \cite{NIPSDistributedPure2,DistributedMABInfocom1,
DistributedMABIJCAI1,DistributedMABICLR1,
FOCSDistributedPure1,DistributedMABCDC1,P2PECC1,
PeterAutom,DMABSto1,DMABSto2} study the \textit{cooperative} multi-agent model that allows agents to communication with other agents. Those models can be roughly categorized into two types. On the one hand, the authors of \cite{NIPSDistributedPure2,DistributedMABInfocom1,
DistributedMABIJCAI1,DistributedMABICLR1,
FOCSDistributedPure1} allows each pair of agents to communicate in a predefined round, and the regret is improved by some $V$-dependent factors, e.g., $1/V$ or $1/\sqrt{V}$. This type of model is most relevant to ours without corruption, i.e., $C=0$. On the other hand, the works \cite{DistributedMABCDC1,P2PECC1,
PeterAutom,DMABSto1,DMABSto2} restricts the communication over the network, wherein agents can only communicate with their neighbors encoded by an undirected communication graph with an unknown structure. Another branch of works \cite{MyFairBandit,DMABTIT1,DMABTIT2,
DMABINFOCOM3,
DistributedMAB3,DistributedMAB4} investigates a \textit{competitive} setting where the reward becomes zero or the reward is split when two or more agents select the same arm simultaneously. These competitive models are considered for some real-world applications wherein the communication between agents is almost impossible, and sometimes agents need to learn from the collision rather than communicating information. This setting essentially differs from ours which leverage the balance between the regret and the communication.

The adversarial multi-agent MAB problem is first examined in \cite{DistributedMABADVCOLT1} which considers some dishonest agents who do not follow the predefined protocol and may send fake observations. Subsequently, the authors of \cite{DistributedMAB2} study the adversarial setting for multiple agents connected by a general communication graph with potential delays, and prove an averaged regret over all agents.
Then, authors of \cite{DMABADVNIPS2} show that all agents can simultaneously maintain a low individual regret under the same setting of \cite{DistributedMAB2}. Moreover, the work \cite{KeepNeighbor} studies the adversarial multi-agent problem in an asynchronous setting where only some of the agents are active at each round, and agents share feedback in an undirected graph. Unlike our model, all of the above works pessimistically assume that the unknown environment is \textit{completely adversarial}, and thus achieve a $\tilde{O}(\sqrt{T})$ (ignoring factors of delay and network properties) bound even if the environment is completely stochastic. As a result, those cooperative (coop)-\textsc{Exp3} algorithms cannot obtain a better $O(\log T)$ bound when the environment is close to being stochastic (i.e., under moderate corruptions).
This paper allows the environment to adaptively transit between the stochastic one and the adversarial one, as the adversary to inject the different amount of corruption.

\textbf{Stochastic Bandits with Corruption.} 
Lykouris et al. \cite{STOCMultiLayer} first introduce the corrupted setting and incorporate the corruption level $C$ in the regret bound as $\widetilde{O}(\sum_{i \neq i^*}C/\Delta_i)$. Subsequently, the work \cite{COLT19Gupta} improves the regret bound of \cite{STOCMultiLayer} and derives a $O(KC)+\widetilde{O}(\sum_{i \neq i^*}1/\Delta_i)$ bound which enjoys an additive term of corruption $C$. The authors of \cite{FIXEDCORRUPTION} consider a probabilistic corruption model where the corruption in the current round is independent of previous corruption. Even though authors \cite{FIXEDCORRUPTION} obtain the $O(C)+\widetilde{O}(\sum_{i \neq i^*}1/\Delta_i)$ bound, but the algorithm of \cite{FIXEDCORRUPTION} is not robust to the attack policy which does not follow some specific probabilistic models. The works
\cite{JMLRTsallis,ImprovedTsallis} adapt online mirror descent (OMD) technique with Tsallis entropy to achieve the expected regret of $O(\mathbb{E}[C])+\widetilde{O}(\sum_{i \neq i^*}1/\Delta_i)$ in the corrupted MAB problem. Note that the works \cite{JMLRTsallis,ImprovedTsallis} use a different regret metric, and the additive term of $\mathbb{E}[C]$ is necessary if we convert their results to $\mathbb{E}[R_T]$ (see Appendix \ref{discussofTsallis}).
Besides, the corrupted setting is also considered for the linear bandits \cite{YingkaiLiLinearCorr,linearCorruption1,
linearCorruption2,LHPcorruption1}. The works \cite{YingkaiLiLinearCorr,linearCorruption2} incurs a $\widetilde{O}(C^2)$ or $\widetilde{O}(C/\Delta_{\min})$ dependence of corruption for regret when the adversary exactly knows the currently chosen arm, while \cite{LHPcorruption1} provides a high probability regret which has additive dependence of $O(C)$, when the adversary is unaware of the current choice. The result in \cite{LHPcorruption1} becomes
$O(C)+\widetilde{O}(K^{1.5}/\Delta_{\min})$ in the MAB setup.
All of the above works consider the corruption for single-agent MAB problem, little is known about the corruption in the multi-agent MAB problem. Further, extending existing single-agent algorithms toward multi-agent ones may fail to balance the communication-regret trade-off due to some technical challenges as discussed in Subsection \ref{Dischallenge}.

\subsection{Contributions} \label{sec:contributions}
Our main contributions are as follows:
\begin{itemize}
\item This paper, to the best of our knowledge, first studies the adversarial corruption problem for the multi-agent setting. For this problem, we present a lower bound $\mathbb{E}[R'_T]=\Omega (\mathbb{E}[C]-\sqrt{VT\log(VT)})$ and $\mathbb{E}[R_T]=\Omega (\mathbb{E}[C]+\mathbb{E}[R'_T])$ for any algorithm with $V$ agents (Theorem \ref{theorem2}) where $R'_T  =   \max_{i \in [K]} \sum_{t=1}^T \sum_{v=1}^V \left(r_{v,i}(t)- r_{i_v(t)}(t) \right)$ and $R_T$ in \eqref{reg}. Different from previous results of \cite{STOCMultiLayer,BestofBoth3} that only give the lower-bound of $\mathbb{E}[R'_T]$, we present the lower bounds of both $\mathbb{E}[R'_T]$ and $\mathbb{E}[R_T]$ for any bandit algorithm with $V \geq 1$. Our results show that the additive term $\mathbb{E}[C]$ is unavoidable for $\mathbb{E}[R_T]$ and $\mathbb{E}[R'_T]$.

\item We propose a new algorithm (Algorithm \ref{alg1}) with $V \geq 1$ agent(s), which is agnostic to the corruption level $C$. The proposed algorithm achieves a high probability regret $R_T=O(VC)+\widetilde{O}(K/\Delta_{\min})$ and expected regret $\mathbb{E}[R_T]=O(V\mathbb{E}[C])+\widetilde{O}(K/\Delta_{\min})$, while maintaining a $O(K\log T)$ communication cost (Theorem \ref{theorem1}). The proposed algorithm is applicable for the single-agent setting, but existing robust single-agent algorithms cannot be trivially extended to the multi-agent setting (see discussion \ref{Dischallenge}). Our algorithm achieves a high probability regret $O(C)+\widetilde{O}(K/\Delta_{\min})$ for $V=1$, which resolves an open question from \cite{COLT19Gupta}.

\item We further discuss the results evaluated by some other possible regret metrics.
We show that any algorithm in the corrupted problem will suffer $\mathbb{E}[R_T] = \Theta(   \mathbb{E}[R'_T]+\mathbb{E} \left[ C\right] )$ and $\mathbb{E}[R_T] = \Theta( \overline{R}_T  +\mathbb{E} \left[ C\right])$ (Theorem \ref{ConnectRegret}) where $\overline{R}_T =  \max_{i } \mathbb{E} [ \sum_{t=1}^T  \sum_{v=1}^V ( r_{v,i}(t)- r_{i_v(t)}(t)  ) ]$ is the (pseudo)-regret defined in the adversarial regime, which cannot be trivially converted to (pseudo)-regret $R_T$ defined in the stochastic regime, when $C \neq 0$. 
\end{itemize}

\section{Lower Bound}
\label{sec:lb}
In this section, we present the lower bound for this problem. 
Since it is challenging to directly lower-bound $R_T$ in the corrupted setting, we first lower-bound another regret notation and then bridge both metrics by considering their expectation.
Let define the regret as $R'_T$ which measures the difference between the reward from the arm that contributes to the maximum cumulative rewards, and the rewards observed by the agents in hindsight:
\begin{equation}\label{actualreg}
R'_T  =   \max_{i \in [K]} \sum_{t=1}^T \sum_{v=1}^V \left(r_{v,i}(t)- r_{i_v(t)}(t) \right),
\end{equation}

 Note that $R'_T$ can be either positive or negative, whereas the pseudo-regret $R_T$ is always non-negative because it measures the difference of means (i.e., the expectation of \textit{stochastic} rewards).

\begin{theorem} \textbf{(Lower bound)} \label{theorem2}
For any $V \geq 1$, if a bandit algorithm with $V$ agent(s) that guarantees a pseudo-regret bound as
$
 O(\log(VT)/\Delta  )
$
in stochastic environments with two arms and $|\mu_1-\mu_2|=\Delta \in (1/4,1/2)$ for sufficiently large $T$, then, there is an corrupted instance with corruption level $C=VT^\alpha$ for $\alpha \in (0,1)$ such that with at least a constant probability $R'_T=\Omega (C-\sqrt{VT\log(VT)} )$. Then, $\mathbb{E}[R'_T]=\Omega (\mathbb{E}[C]-\sqrt{VT\log(VT)} )$ and $ \mathbb{E}[R_T]=
\Omega (  \mathbb{E}[C] + \mathbb{E}[R'_T])
$.
\end{theorem}

\begin{remark}
Theorem \ref{theorem2} shows that the additive term of expected corruption $\mathbb{E}[C]$ is unavoidable in $ \mathbb{E}[R_T]$ for any algorithm that can achieve a near-optimal regret in the stochastic setting. The lower bound of $R'_T$ appear in some previous works for the single-agent corruption setting, e.g.,\cite{STOCMultiLayer,BestofBoth3}, but the lower bound $ \mathbb{E}[R_T]$ is not given. One can observe that both notations cannot be trivially connected since when $C \neq 0$, the expectation over the corrupted reward of $r_{v,i}(t)$ may not yield the mean $\mu_i$.
Note that the lower bound of $ \mathbb{E}[R_T]$ does not imply a lower bound of $R_T$, and the lower bound of $R_T$
remains as an open question.
\end{remark}

\section{Our Algorithm}
\label{Sec3}
% !TEX root = main.tex

This section presents our algorithm in Algorithm \ref{alg1}.
The algorithm proceeds in epochs indexed by $\tau$ whose length $N(\tau)$ with exponential increase depending on a \textit{global error level} $\epsilon(\tau)$ (line \ref{def:Ntau}). The algorithm maintains a \textit{global} active arm set $\cA(\tau)$, and a \textit{global} bad arm set $\cB(\tau)$ such that $[K]=\cA(\tau) \cup \cB(\tau)$. Besides, each agent $v$ is endowed with an \textit{agent-specific} arm set $\cK_v(\tau) \subset [K]$, an \textit{agent-specific} active arm set $\cA_v(\tau)$, and an \textit{agent-specific} bad arm set $\cB_v(\tau)$ such that $\cK_v(\tau)=\cA_v(\tau) \cup \cB_v(\tau)$ and each agent has the same size of $\cK_v(\tau)$, i.e., $|\cK_v(\tau)|=\widetilde{K}$.
We further have that $\cup_{v=1}^V \cA_v(\tau)=\cA(\tau)$ and $\cup_{v=1}^V \cB_v(\tau)=\cB(\tau)$.
Each arm $i$ is associated with an arm-specific error level $\epsilon_i(\tau)$. The error level $\epsilon_i(\tau)$ controls the expected number of pulls of arm $i$ in epoch $\tau$, denoted by $N_{v,i}(\tau)$.

At the beginning of epoch $\tau$, Algorithm \ref{alg1} runs \textsc{Arm Allocation} subroutine to divide the arm space into total $V$ parts and thus each agent $v$ only needs to learn a small arm space $\cK_v$ (line \ref{algArmAllo}).
In each round, every agent $v \in [V]$ samples an arm $i \in \cK_v(\tau)$ with the probability $p_{v,i}(\tau)$ (line \ref{algPullArm}), given by
\begin{equation} \label{pullprob}
p_{v,i}(\tau)=   \left\{
\begin{aligned}
&\frac{\epsilon^2(\tau)}{\epsilon^2_{i}(\tau)\widetilde{K}},  &  \quad i \in  \cB_v(\tau), \\
&\left(1-\sum_{j \in \cB_v(\tau)}p_{v,j}(\tau) \right) \frac{1}{\left|\cA_v(\tau) \right|},  &  \quad i \in  \cA_v(\tau).
\end{aligned}
\right.
\end{equation}

At the end of epoch $\tau$, the algorithm first proceeds in Step 1 in which each agent $v$ updates the estimator $\hat{\mu}_{v,i}(\tau) $ for all $i \in \cK_v(\tau)$, and then broadcasts the arm sets i.e., $\cA_v(\tau)$ and $\cB_v(\tau)$ and estimation to the leader agent (line \ref{algComm}). The leader agent constructs a global estimator $\hat{\mu}_{i}(\tau)$ for each arm $i \in [K]$ by averaging estimators over all agents. Then, the leader agent proceeds in Step 2 which \textit{reactivates} the seemingly-good arms from $\cB(\tau)$ whose estimations are inconsistent with previous ones (line \ref{algReact}), and then \textit{deactivates} seemingly-bad arms which contribute to less reward in this epoch (line \ref{algDeact}).
Finally, Step 3 updates the global error level $\epsilon(\tau)$ and arm-specific error level $\epsilon_i(\tau)$ for the next epoch (line \ref{algErrorBegin}-\ref{algErrorEnd}) according to the results of Step 2. We then explain these three important ingredients of our algorithm, including arm allocation, robust reactivation and deactivation, and error level update in detail.

\begin{algorithm}[t] 
\caption{Cooperative Bandit Algorithm Robust to Adversarial Corruptions}
\label{alg1}
\begin{algorithmic}[1] 
\Require Time horizon $T$, and confidence $\delta \in (0,1)$.
\State Set $\tau=1$, $N(\tau)=0$, and $\epsilon(\tau)=\epsilon_{i}(\tau)=\frac{1}{14}$ for all $i$. Randomly sample $\hat{i}^*$ from $[K]$.
\State Set $\cA(\tau)=[K]$, $\cB(\tau)=\emptyset$, and $\mathcal{H}(\tau)=\emptyset$ for $\tau=1$. Set $d_i=1$ for all $i$.
\State Select a leader agent arbitrarily or according to any predefined strategy. 
\State \underline{Warm-up: Arm allocation:}
\State Run Algorithm \ref{alg2} to get $\widetilde{K}$, $\cK_v(\tau)$, and $\cA_v(\tau)$, and $\cB_v(\tau)$. \label{algArmAllo}
\For{$\tau =1,2,3,...$}
\State Set $ N(\tau)=3\widetilde{K} \log((8K \log_4 T)/\delta)/\epsilon^2(\tau)$, and $T(\tau)=T(\tau-1)+N(\tau)$. \label{def:Ntau}
\For{ $t=T(\tau-1)+1$ to $T(\tau)$}
\State Each agent $v$ pulls an arm $i_v(t)$ with probability 
$
p_{v,i}(\tau)$ and observes reward $r_{i_v(t)}(t)$. \label{algPullArm}
\EndFor
\State \underline{Step 1: Parameter estimation:} (run by each agent)
\State Update estimation $\hat{\mu}_{v,i}(\tau)$ for each agent $v$ and arm $i$ in epoch $\tau$.
\begin{equation*}
\hat{\mu}_{v,i}(\tau) = \frac{\sum_{t \in \mathcal{T}(\tau)}r_{i_v(t)}(t) \mathbb{I}\{ i_v(t)=i\}}{N_{v,i}(\tau)},
\end{equation*}
\State where $N_{v,i}(\tau)=p_{v,i}(\tau)N(\tau)$ and  $\cT(\tau)=\{t:T(\tau-1)+1\leq t \leq T(\tau)\}$.
\State Each agent $v$ broadcasts $ \cA_v(\tau)$, $ \cB_v(\tau)$, and $  \{ \hat{\mu}_{v,i}(\tau)\}_{i \in \cK_v(\tau) }$ to the leader agent. \label{algComm}
\State The leader agent updates estimation $\hat{\mu}_{i}(\tau) $ for each arm $i$ in epoch $\tau$.
\begin{equation*}
\hat{\mu}_{i}(\tau) =\frac{1}{|\{v \in [V]:i \in \cK_v(\tau) \}|} \sum_{v \in [V]:i \in \cK_v(\tau)}\hat{\mu}_{v,i}(\tau).
\end{equation*}
\State \underline{Step 2: Robust reactivation and deactivation:} (run by the leader agent)
\State \textbf{Reactivation:} Identify $\mathcal{H}(\tau)=\left\{ i \in \cB(\tau): \max\limits_{j \in \cA(\tau)} \hat{\mu}_{j}(\tau) -\hat{\mu}_{i}(\tau)< 4\epsilon(d_i) \right\}$ . \label{algReact}
\State Set $\hat{\mu}^*(\tau)=\max\limits_{j \in \cA(\tau) \cup \mathcal{H}(\tau)} \left\{ \hat{\mu}_{j}(\tau) +2\epsilon_{j}(\tau) \right\}$ and find $\hat{i}^*$ such that $\hat{\mu}^*(\tau)= \hat{\mu}_{\hat{i}^*}(\tau) +2\epsilon_{\hat{i}^*}(\tau) $. 
\State \textbf{Deactivation:} Identify $\mathcal{M}(\tau)=\left\{ i \in \cA(\tau) \cup \mathcal{H}(\tau):\hat{\mu}^*(\tau) -\hat{\mu}_{i}(\tau)> 14\epsilon(\tau) \right\}$. \label{algDeact}
\State Set $\cA(\tau+1) = \left(\cA(\tau)\cup \mathcal{H}(\tau) \right) \backslash \mathcal{M}(\tau)$, and $\cB(\tau+1) =\left( \cB(\tau) \backslash \mathcal{H}(\tau) \right) \cup \mathcal{M}(\tau)$.
\State \underline{Step 3: Error level update:} (run by the leader agent)
\State Set $\epsilon(\tau+1) =\epsilon(\tau) /2$. \label{algErrorBegin}
\State Set $\epsilon_{i}(\tau+1) =\epsilon(\tau+1) $ and $d_i=\tau$ for $\forall i \in \cA(\tau+1)$.
\State Set $\epsilon_{i}(\tau+1) =\epsilon(d_i) $ for $\forall i \in \cB(\tau+1)$.\label{algErrorEnd}
\State Leader broadcasts $\hat{i}^*$ for each agent $v$, and all agents update $\hat{i}^*$. \label{algComm3}
\State Leader broadcasts $\epsilon_{i}(\tau+1)$ for $i \in \cK_v(\tau+1)$, $\cA_v(\tau+1)$, and $\cB_v(\tau+1)$ for each agent $v$. \label{algComm2}
\EndFor
\end{algorithmic}
\end{algorithm}

\renewcommand{\algorithmicrequire}{\textbf{Input:}}
\renewcommand{\algorithmicensure}{\textbf{Return:}}

\begin{algorithm}[htp!] 
\caption{\textsc{Arm Allocation}}
\label{alg2}
\begin{algorithmic}[1] 
\Require Agent number $V$, arm number $K$, $\hat{i}^*$, $\cA(\tau)$, and $\cB(\tau)$.
\State Set $\widetilde{K} =\lceil K/V \rceil+1$. Find a minimum $\bar{v} \in [V]$ such that $\bar{v}\lceil K/V \rceil \geq K$. 
\State Set $\cK_v(\tau)=\{(v-1)\lceil K/V \rceil+1,...,v\lceil K/V \rceil \}$ for $v=1,...,\bar{v}-1$, if $\bar{v} \geq 2$. For $v=\bar{v}$, we first set $\mathcal{K}_{\bar{v}}(\tau)=\{(v-1)\lceil K/V \rceil+1,...,K \}$ and then randomly sample $\bar{v} \lceil K/V \rceil-K$ arms from $[K] \backslash \mathcal{K}_{\bar{v}}(\tau)$ to merge them to obtain a new $\mathcal{K}_{\bar{v}}(\tau)$. For those $v=\bar{v}+1,...,V$, if exist, we randomly sample $\lceil K/V \rceil$ arms for them.
\State Update ${K}_v(\tau)={K}_v(\tau) \cup \{\hat{i}^* \}$, $\cA_v(\tau)=\cA(\tau) \cap \cK_v(\tau)$, and $\cB_v(\tau)=\cB(\tau) \cap \cK_v(\tau)$.
\Ensure $\widetilde{K}$, $\cK_v(\tau)$, $\cA_v(\tau)$, and $\cB_v(\tau)$.
\end{algorithmic}
\end{algorithm}

\textbf{Arm allocation.} The main idea of arm allocation is to divide the arm space into several parts which are allocated with agents, and then each agent only needs to (i) learn a small arm space, and (ii) broadcast a small message with $O(K/V)$ size. The details of arm allocation is presented in Algorithm \ref{alg2}. Each agent $v \in [V]$ is allocated with exactly $\widetilde{K}=\lceil K/V \rceil+1$ arms such that $\cup_{v=1}^V \cK_v(\tau)=[K]$, and each agent has the same number of arms. Note that we need $\cK_v(\tau)$ to have at least one active arm $i$ such that $i \in \cA(\tau)$, and otherwise $\sum_{i \in \cK_v} p_{v,i}(\tau)\neq1$.
To this end, we add an additional arm $\hat{i}^*$ named \textit{empirical best arm} for each $ \cK_v(\tau)$. The empirical best arm is given by 
\begin{equation}  \label{seemingbestarm}
\hat{i}^*=\arg\max_{j \in \cA(\tau) \cup \cH(\tau)}  \hat{\mu}_{j}(\tau) +2\epsilon_{j}(\tau) ,
\end{equation}
whose corresponding estimation is denoted by $\hat{\mu}^*$. Adding arm $\hat{i}^*$ can ensure that $\sum_{i \in \cK_v} p_{v,i}(\tau)=1$ because the arm $\hat{i}^*$ identified in epoch $\tau $ must be active in epoch $\tau +1$ according to the deactivation rule (regard the initialization as epoch $\tau=0$).

The similar idea of dividing arm space appears in, for example \cite{NIPSDistributedPure2,
FOCSDistributedPure1,DistributedMABICLR1}
for multi-agent bandits. Note that compared with the ones in previous works, our algorithm has a unique difference. The algorithms in \cite{NIPSDistributedPure2,
FOCSDistributedPure1,DistributedMABICLR1}
only allocate active arms for agents, thereby the arm space shrinking in proceeding epochs, but our algorithm needs to ensure that all arms have chance to be pulled for every epoch due to the presence of the adversary. As a consequence, we do not shrink the arm space, but instead shrink the probability of pulling arm. This idea will be clear in the analysis of Step 3, error level update.

\textbf{Robust reactivation and deactivation.} 
This design is motivated by the fact that an adversary might trick traditional active arm elimination (AAE) methods \cite{Median,SAR,STOCMultiLayer} to eliminate the optimal arm \textit{permanently} in the corrupted setting. To address this issue, we allow those deactivated (eliminated) arms to be reactivated again. Specifically, the algorithm first checks whether the estimators of some bad arms in $\cB(\tau)$ suddenly become better than those of arms in $\cA(\tau)$. If there exist such arms, we put them into a temporary set $\mathcal{H}(\tau)$. An interesting observation is that the reactivation will not negatively impact the regret performance in the stochastic setting without corruption. This is due to the following lemma that $ \mathcal{H}(\tau)$ with high probability is always an empty set, and thus reactivation will not be triggered by the algorithm.
\begin{lemma} \label{lemmaforStochastic}
In the stochastic setting, with probability at least $1-\delta$, $ \mathcal{H}(\tau)=\emptyset$ for all $\tau$.
\end{lemma}

The algorithm does not directly reactivate those arms in $ \mathcal{H}(\tau)$ but instead puts them together with arms in $\cA(\tau)$ for the deactivation step.
In the deactivation step, the algorithm compares the estimators of $\hat{i}^*$ with each arm $i \in \cA(\tau) \cup \mathcal{H}(\tau)$. One can see that not all arms in set $\mathcal{H}(\tau)$ will be successfully reactivated, but the algorithm only reactivates those arms that suffice
\begin{equation}\label{check1}
\left\{ i \in \cA(\tau) \cup \mathcal{H}(\tau):\hat{\mu}^*(\tau) -\hat{\mu}_{i}(\tau) \leq 14\epsilon(\tau) \right\}.
\end{equation}

Adding arms in $\mathcal{H}(\tau)$ for deactivation forces the adversary to inject more corruption. Specifically, if an adversary hopes to reactivate some target arms, the deactivation step forces the adversary to inject large enough corruption to guarantee that the estimators of target arms are close to $\hat{\mu}^*$ so that they can be \textit{successfully} reactivated.

\textbf{Error level update.} The high-level idea of this design is to use the error level to control the expected number of pulls of arm. The arm $i$ with a large gap $\Delta_i$ will be assigned with a large error level and thus it will be pulled a few times in the epoch. Our algorithm maintains a global error level $\epsilon(\tau)$ and an arm-specific $\epsilon_i(\tau)$. For arms $i \in \cA(\tau)$, we set $\epsilon_i(\tau)=\epsilon(\tau)$, and for $i \in \cA(\tau)$, set $\epsilon_i(\tau)=\epsilon(d_i)$ where $d_i$ is the \textit{last epoch} up to the current epoch such that arm $i$ holds \eqref{check1}.

The expected number of pulls of a bad arm $i$ in agent $v$ is $N_{v,i}(\tau)=p_{v,i}(\tau)N(\tau)$, which depends on $1/\epsilon_i^2(d_i)$. Hence, if the arm is deactivated earlier, then, it will be pulled less in expectation. This further implies that in the stochastic setting, if the arm $i$ has a large arm gap $\Delta_i$, it will be deactivated early so that it incurs less regret. 
\begin{remark}
Algorithm \ref{alg1} can naturally reduce to a single-agent one without any additional procedure. In this case,
one can drop all $v$ for those parameters or random variables, e.g., $\hat{\mu}_{v,i}(\tau)$ and $N_{v,i}(\tau)$. Then, we have that $\cA_v(\tau)=\cA(\tau)$, $\cB_v(\tau)=\cB(\tau)$, and $\cK_v(\tau)=[K]$, and thus the leader agent runs the entire algorithm.
\end{remark}

\section{Regret and Communication Analysis}
\label{Sec4}
% !TEX root = main.tex

%\begin{equation*}
% O\left(VC+ \log((VK\log_4T)/\delta)  \sum_{\tau=1}^L \left( \sum_{v=1}^V \sum_{i \in \mathcal{A}_v(\tau) \wedge i \neq i^*}   \frac{ \widetilde{K} }{  \left| \mathcal{A}_v(\tau)\right| \Delta_{i}} + \sum_{v=1}^V \sum_{i \in \mathcal{B}_v(\tau) \wedge i \neq i^*}   \frac{  1}{ \Delta_{i}}  \right)   \right).
%\end{equation*}
%
%The regret can be rewritten in a concise form as 

We now provide the main results of nearly instance-optimal regret bound and communication cost of Algorithm 1.

\begin{theorem} \label{theorem1}
\textbf{(Upper bound)} Algorithm \ref{alg1} which is agnostic to the corruption level $C$, with probability $1-\delta$, incurs $O(K \log T)$ communication cost and regret as
\begin{equation*}
 O\left(VC+\frac{ K \log T \log((VK\log T)/\delta) }{  \Delta_{\min}}  \right),
\end{equation*}
and incurs the expected regret $\mathbb{E}[R_T]$ as 
$
 O(V\mathbb{E}[C]+\frac{ K \log T \log(VT) }{  \Delta_{\min}}  ).
$
\end{theorem}

\begin{remark}
Although $\widetilde{O}(K/ \Delta_{\min}  )$ in Theorem \ref{theorem1} is slightly weaker than $\widetilde{O}(\sum_{i \neq i^*} 1/ \Delta_{i }  )$,
the latter one is also controlled by $\Delta_{\min}$.
Note that our regret can be also written in a summation form with some algorithm-dependent random variables, i.e., summing over active arms $i \in \mathcal{A}_v(\tau)$, and over bad arms $i \in \mathcal{B}_v(\tau)$, respectively. The regret of summation form can be found in Appendix \ref{summationformReg}. 
\end{remark}

Theorem \ref{theorem1} shows that in the uncorrupted setting, i.e., $C=0$, our regret bound recovers the near-optimal regret of \cite{DistributedMABICLR1} up to a logarithmic factor, while maintaining efficient communication with a linear dependence on $K$ and a logarithmic dependence on $T$.
One can also see that in the uncorrupted setting, our instance-dependent bound only has a $\log(V)$ dependence of agent number, which implies that the algorithm enjoys the speedup of learning.

Theorem \ref{theorem1} also reveals that in the single-agent setting, i.e., $V=1$, the proposed algorithm can achieve a high probability regret as $ O(C)+\widetilde{O}( K / \Delta_{\min}  )$. 
Our expected regret bound $\mathbb{E}[R_T]$ matches the state-of-the-art regret bounds $\widetilde{O}(\mathbb{E}[C]+K/\Delta_{\min}+\sqrt{K\mathbb{E}[C]/\Delta_{\min}})$ e.g., \cite{JMLRTsallis,ImprovedTsallis} (up to a logarithmic factor). 
Recall that the lower bound in Theorem \ref{theorem2} shows that $ \mathbb{E}[C] $ is unavoidable for any algorithm, which corroborates that our regret bound is tight. Note that the regret bounds in \cite{JMLRTsallis,ImprovedTsallis} are given as
$\widetilde{O}(K/\Delta_{\min}+\sqrt{K\mathbb{E}[C]/\Delta_{\min}})$ instead of $\widetilde{O}(\mathbb{E}[C]+K/\Delta_{\min}+\sqrt{K\mathbb{E}[C]/\Delta_{\min}})$. This is because they use the regret metric defined for adversarial regime, but we use the regret $R_T$ following other single-agent corruption models \cite{COLT19Gupta,FIXEDCORRUPTION,
LHPcorruption1}.
We discuss these regret notions in Subsection \ref{discussofReg} and Appendix \ref{discussofTsallis}.

\paragraph{Proof Sketch}
In this subsection, we provide a proof sketch for regret bound in Theorem \ref{theorem1}, and show the way to resolving the open question from \cite{COLT19Gupta}, i.e., removing the multiplicative dependence of $K$ on $C$. Before sketching the proof, let $\tilde{N}_{v,i}(\tau)$ be the real number of pull of arm $i$ by agent $v$ in epoch $\tau$ and define $C(\tau)=\max_{i \in [K]}\sum_{v=1}^V\sum_{t \in \mathcal{T}(\tau)}| r_{v,i}(t)- r_{v,i}^S(t)|$ (if arm $i$ is not allocated to agent $v$, then, $| r_{v,i}(t)- r_{v,i}^S(t)|=0$). Then, the following lemma presents a good event $\mathcal{E}$, showing that $\hat{\mu}_{v,i}(\tau)$ and $\tilde{N}_{v,i}(\tau)$ are close to their actual expectations, respectively.
\begin{lemma}\label{Tmpconcentration2All}
Let define event $\mathcal{E}$ as
\begin{equation} \label{Tmpconcentration}
\mathcal{E}= \left\{\forall v,i,\tau: \left| \hat{\mu}_{i}(\tau)-\mu_i \right|   \leq  \mu_i +2\epsilon_{i}(\tau) + \frac{2 C(\tau)}{N(\tau)}, \tilde{N}_{v,i}(\tau)\leq 3N_{v,i}(\tau) \right\}.
\end{equation}

Then, we hold that $\mathbb{P}\left[\mathcal{E}\right] \geq 1-\delta$.
\end{lemma}

We rewrite the regret as $R_T=\sum_{\tau}\sum_{v}(\sum_{i \neq i^*,i \in \cA(\tau)}\Delta_{i}\tilde{N}_{v,i}(\tau)+\sum_{i \neq i^*, i \in \cB(\tau)}\Delta_{i}\tilde{N}_{v,i}(\tau)) $. Under event $\mathcal{E}$, $\tilde{N}_{v,i}(\tau) \Delta_{i}$ can be upper-bounded by $3R_{v,i}(\tau)$ where $R_{v,i}(\tau)= N_{v,i}(\tau) \Delta_i$. In the following, we turn to bound every $R_{v,i}(\tau)$.

The \textit{main idea} is to bound $R_{v,i}(\tau)$ by considering two cases including $\Delta_i > 32\eta(d_i) $ and $\Delta_i \leq 32\eta(d_i) $ where $\eta(d_i) =\sum_{s=1}^{d_i} \frac{C(s)}{2^{4d_i-4s}N(s)}$.
On the one hand, $\Delta_i > 32\eta(d_i) $ implies a small volume of corruption. In this case, $R_{v,i}(\tau)$ is bounded by $O(1/\Delta_i)$. On the other hand, $\Delta_i \leq 32\eta(d_i) $ implies a large volume of corruption on arm $i$ in agent $v$ during epoch $\tau$. In this case, the regret from all $R_{v,i}(\tau)$ is bounded by $O(VC)$.

\textbf{Case 1: $\Delta_i \leq 32 \eta(d_i) $.} 
This case uses $\Delta_i \leq 32 \eta(d_i) $ to bound $R_{v,i}(\tau) \leq 32N_{v,i}(\tau)\eta(d_i) =O(VC)$. Although \cite{COLT19Gupta} uses a similar construction of $\eta(d_i)$, their algorithm suffers a multiplicative dependence of $K$. As a result, extending their algorithm for the multi-agent case might incur an additive term as $O(VKC)$. On the contrary, our algorithm can remove the dependence of $K$ thanks to the differential treatment of active arm $i \in \cA(\tau)$ and bad arm $i \in \cB(\tau)$. Specifically, for $i \in \cA(\tau)$, we have a factor $1/|\cA(\tau)|$ in $N_{v,i}(\tau)$, which offsets the summation $\sum_{i \neq i^*,i \in \cA(\tau)} 1$.  As for $i \in \cB(\tau)$, $N_{v,i}(\tau)$ has a factor $1/K$ which offsets the summation $\sum_{i \neq i^*,i \in \cB(\tau)} 1$.

\textbf{Case 2: $\Delta_i > 32\eta(d_i) $.} 
Note that the condition $\Delta_i > 32\eta(d_i) $ directly bounds the corruption term in \eqref{Tmpconcentration} as
\begin{equation*} 
\Delta_i > 32\eta(d_i)  = \sum_{s=1}^{d_i} \frac{32C(s)}{2^{4d_i-4s}N(s)} \geq  \frac{32C(d_i)}{N(d_i)}.
\end{equation*}

Then, we can bound $R_{v,i}(\tau)$ for this case by making use of the concentration of estimators and some algorithm properties. Here, we provide some properties of Algorithm \ref{alg1} that are used for the proof sketch. All properties can be found in Appendix \ref{algprop}.
\begin{lemma} \label{property}
For any $\tau$, Algorithm \ref{alg1} holds that
(i) if $i \in \mathcal{A}(\tau)$, then, $\epsilon(\tau)=\epsilon_i(\tau)=\epsilon(d_i)/2$;
(ii) if $i \in \mathcal{B}(\tau)$, then, $\epsilon_i(\tau)=\epsilon(d_i)$.
\end{lemma}

As $R_{v,i}(\tau)= N_{v,i}(\tau) \Delta_i$, we thus turn to bound $N_{v,i}(\tau)$. Recall that $N_{v,i}(\tau)=p_{v,i}(\tau)N(\tau)=\widetilde{O}(1/\epsilon^2_i(\tau))$. Further, Lemma \ref{property} connects $\epsilon(d_i)$ and $\epsilon_i(\tau)$ as 
$\epsilon_i(\tau)=\epsilon(\tau)=\epsilon(d_i)/2$ if arm $i \in \cA(\tau)$, and $\epsilon_i(\tau)=\epsilon_i(d_i)$ if arm $i \in \cB(\tau)$. Thus, we have that $N_{v,i}(\tau)=\widetilde{O}(1/\epsilon^2(d_i))$.
One can see that we hope to prove $\epsilon(d_i) \geq \Delta_i$ to obtain a bound $N_{v,i}(\tau)=\widetilde{O}(1/\Delta^2_i)$. The following lemma gives the desired result.
\begin{lemma} \label{lemma2Delta}
If $\Delta_{i} > 32\eta(d_i)$, then, with probability at least $1-\delta$, $\epsilon(d_i) \geq \Delta_i/32$.
\end{lemma}

Lemma \ref{lemma2Delta} makes use of robust reactivation and deactivation by considering cases including $i^* \in \mathcal{A}(d_i) \cup \mathcal{H}(d_i)$ and $i^* \in \mathcal{B}(d_i) \backslash \mathcal{H}(d_i)$. On the one hand, if $i^* \in \mathcal{A}(d_i) \cup \mathcal{H}(d_i)$, we use the property of deactivation, i.e., the optimal arm $i^*$ must have \eqref{check1}, which leads to $\epsilon(d_i) \geq \Delta_i/32$. For $i^* \in \mathcal{B}(d_i) \backslash \mathcal{H}(d_i)$, we use the property that the optimal arm $i^*$ is not reactivated at the beginning of epoch $d_i$, which implies that $ 4\epsilon(d_{i^*})  \leq \max_{j \in \mathcal{A}(d_i)} \hat{\mu}_{j}(d_i) -\hat{\mu}_{i^*}(d_i) $. From this, we can get that $\epsilon(d_i) \geq \Delta_i/32$. Combing the above analysis, we complete the proof of regret bound. The details of this proof can be found in Appendix \ref{Appendixregret}.

\section{Discussion}
\label{sec:exten}
% !TEX root = main.tex

We conclude this paper with some discussions of our model and results.

\subsection{Challenges of Extending Single-agent toward Multi-agent} \label{Dischallenge}
We here present two challenges that make the adversarial corruption problem for the multi-agent MAB setup non-trivial. The first challenge is that the standard concentration inequalities for \textit{dependent} random variables cannot be directly applied for the multi-agent case. Concretely, the process that all agents simultaneously observe $V$ realized values of rewards for $T$ rounds cannot be simulated by the process that a single agent sequentially observes a realization of reward for $VT$ rounds, when the realizations of rewards depend on the previous history, e.g., corruptions and arm selections.
Extending the standard concentration bounds for the multi-agent setting requires a non-trivial analysis \cite{P2PECC1,PeterAutom}.

The other challenge is that simply extending existing single-agent robust algorithms, e.g., \cite{COLT19Gupta,JMLRTsallis,ImprovedTsallis,LHPcorruption1}
may fail to balance the communication-regret trade-off.
For example, the algorithms of \cite{JMLRTsallis,ImprovedTsallis} need agents to share observations and update estimators in each round, which yields total $O(VKT)$ communication cost.
A straightforward idea of reducing the communication cost is to adapt the epoch/phase-based algorithms \cite{COLT19Gupta,LHPcorruption1} for multi-agent setting. One may divide the epoch into $V$ parts, where $V$ agents simultaneously proceed in a smaller epoch. However, this does not suggest that the communication cost of epoch-based algorithms \cite{COLT19Gupta,LHPcorruption1} is comparable to ours. Catoni estimator used in \cite{LHPcorruption1} requires each agent to broadcast a loss estimator sequence whose size is equal to the epoch length. As a consequence, this yields total $O(KT)$ communication cost. The algorithm \cite{COLT19Gupta} might incur $O(KV\log T)$ communication cost that is $V$ times larger than ours, because each agent needs to share the observations of all $K$ arms for at most $\log T$ epochs. Moreover, the regret of \cite{COLT19Gupta} has a multiplicative dependence of $K$, and thus the corruption term in regret might be $O(VKC)$ in the multi-agent setting.

\subsection{Corrupted Setting versus Adversarial Setting}
In the corrupted setting, the rewards are assumed to be initially drawn from fixed and unknown distributions, whereas in the adversarial setting, the rewards may not follow any distribution even before the adversarial attack. The corrupted setting is appropriate for those scenarios where agents are more interested in the stochastic structure, i.e., $\mu_i$, than the actually observed reward, i.e., $r_{v,i}(t)$ that is critical in the adversarial setting. For instance, the platform aims to recommend ads that maximize the preference of users, and thus the user's preference (i.e., $\mu_i$) is what the platform cares about. As a consequence, the algorithms of the corrupted model \cite{YingkaiLiLinearCorr,linearCorruption1,
linearCorruption2,COLT19Gupta,LHPcorruption1} are commonly evaluated by the regret defined in the stochastic regime (e.g., $R_T$ or $\mathbb{E}[R_T]$), which is different from that defined in the adversarial regime (this will be clear in the next subsection). Moreover, since the corruption problem concerns the stochastic structure, it is typically assumed that the total corruption level $C$ is moderate, which implies that the environment in the corrupted problem is more close to being stochastic than adversarial.

\subsection{Alternative Regret Notions}\label{discussofReg}
The bandit algorithms in adversarial setting, e.g., \textsc{Exp3} is typically measured by regret as $
\overline{R}_T =  \max_{i } \mathbb{E} [ \sum_{t=1}^T  \sum_{v=1}^V ( r_{v,i}(t)- r_{i_v(t)}(t)  ) ]$. Note that $\overline{R}_T$ is also called pseudo-regret \cite{BanditBook1}, but it is defined in the adversarial regime and is essentially different from pseudo-regret $R_T$ (see \eqref{reg}) defined in the stochastic regime, when $C \neq 0$.
One can see that $\overline{R}_T$ coincides with the $\mathbb{E}[R_T]$ in the uncorrupted setting as $r^S_{v,i}(t)=r_{v,i}(t)$, but $\overline{R}_T$ cannot be trivially converted to $\mathbb{E}[R_T]$ in the presence of an adversary. The following theorem connects notations of $
\overline{R}_T$, $R_T$, and$R'_T$ (see \eqref{actualreg}).
\begin{theorem}\label{ConnectRegret}
For any algorithm of $V \geq 1$ agents, with probability at least $1-1/VT$, we have that
$
R_T  =O(\overline{R}_T+\mathbb{E} \left[ C\right] + V \log(VT) ).$
Then, $\mathbb{E}[R_T] = \Theta(\overline{R}_T  +\mathbb{E} \left[ C\right] )$ and $\mathbb{E}[R_T] = \Theta(   \mathbb{E}[R'_T]+\mathbb{E} \left[ C\right] )$ hold for the two-armed bandit instance.
\end{theorem}

Theorem \ref{ConnectRegret} shows that as for the two-armed bandit instance, the expected regret $\mathbb{E}[R_T]$ in stochastic regime can imply the lower bounds and upper bounds of both pseudo-regret $\overline{R}_T$ and expected regret $\mathbb{E}[R'_T]$ in adversarial regime. The regret $\overline{R}_T$ and $\mathbb{E}[R'_T]$ can in turn give the the lower bound and upper bound of $\mathbb{E}[R_T]$.

\subsection{Open Questions}\label{openquestion}
We show a lower bound of $\mathbb{E}[R_T]$ for our problem, but the lower bound of $R_T$ remains as an open question.
Our work leaves a gap between $\Omega(V\mathbb{E}[C])$ in upper bound and $\Omega(\mathbb{E}[C])$ in lower bound.
Bridging this gap while maintaining an efficient communication is an interesting open question. Designing a multi-agent algorithm to achieve the best of three worlds, i.e., stochastic, corrupted, and adversarial settings, is another compelling question. It is also interesting to consider the multi-agent corruption problem for the constrained communication model wherein agents are connected in an unknown graph and only allowed to communicate with neighbors.

%\section*{Societal Impacts} \label{SocImpact}
%This work aims to provide a theoretical understanding of the adversarial corruption problem for the multi-agent setting. Our algorithm can thus be used as a tool to enhance the robustness to potential corruptions for multi-agent bandit learning.

\bibliographystyle{plainnat}
\bibliography{mybibfile}

\newpage
\appendix

% !TEX root = main.tex

\section{Concentration inequalities}

\begin{lemma}(Hoeffding-Azuma’s inequality for martingales, \cite{BestofBoth1}, Theorem 4.2). \label{HAB}
Let $X_0,X_1,...,X_n$ be a martingale difference sequence with zero mean and $ |X_i-X_{i-1}| \leq m_i$ almost surely for all $i \geq 1$. Then, we have for all $\delta>0$,
\begin{equation} \label{HABineq}
\mathbb{P}\left(\left| \sum_{i}^n X_i  \right| > \sqrt{\frac{\log(2/\delta)}{2}\sum_{i}^n m_i}  \right) \leq \delta .
\end{equation}
\end{lemma}

\begin{lemma}(Freedman's concentration inequality, \cite{Freedman}, Theorem 1). \label{Freedm}
Let $X_1,...,X_n$ be a martingale difference sequence with zero mean and $|X_i| \leq M$ almost surely for all $i$. Let $V=\sum_{i=1}^n \mathbb{E}\left[X_i^2|X_1,...,X_{n-1} \right]$ be the cumulative variance of the martingale. Then, we have for all $\delta>0$,
\begin{equation*}
\mathbb{P}\left( \sum_{i=1}^n X_i >  \frac{V}{M} + M\log (1/\delta)  \right) \leq   \delta .
\end{equation*}
\end{lemma}

\section{Notations}\label{notationtable}
For ease of reading, we list here key notations that will be used in this Appendix.
\begin{center}
\begin{tabular}{r c p{11.5cm}}
	$T, V, K$ &: & Time horizon, agent number, and arm number. \\
		$i^*, \hat{i}^*$ & : & Best arm and empirical best arm. See \eqref{seemingbestarm} for the detail of $\hat{i}^*$. \\
	$\tilde{N}_{v,i}(\tau)$ & : & The real number of times that arm $i$ is pulled by agent $v$ in epoch $\tau$. \\
	$N_{v,i}(\tau)$ & : & The expected number of times that arm $i$ is pulled by agent $v$ in epoch $\tau$. \\
	$p_{v,i}(\tau)$ & : & The probability that arm $i$ is pulled by agent $v$ in epoch $\tau$. \\
	$T(\tau)$ & : & The total number of rounds up to the end of epoch $\tau$, and $T(\tau=0)=0$. \\
	$\mathcal{T}(\tau)$ & : & The set of rounds in epoch $\tau$, and $\cT(\tau)=\{t:T(\tau-1)+1\leq t \leq T(\tau)\}$. \\
	$y_{v,i}(t)$ & : & An independent draw from a Bernoulli distribution with mean $p_{v,i}(\tau)$ for $t \in \mathcal{T}(\tau)$. We have $y_{i}(t)=1$ if $i_v(t)=i$, and $y_{i}(t)=0$, otherwise. \\
	$ \hat{\mu}_{v,i}^S(\tau) $ & : & $ \hat{\mu}_{v,i}^S(\tau) =\sum_{t \in \mathcal{T}(\tau)} r^S_{v,i}(t) \mathbb{I}\{ i_v(t)=i\}/N_{v,i}(\tau)$ is the estimation of arm $i$ computed by \textit{stochastic} rewards from agent $v$ in epoch $\tau$. \\
	$C_{v,i}(t)$ & : & $C_{v,i}(t)=| r_{v,i}(t)- r_{v,i}^S(t)|$ is the corruption on arm $i$ in agent $v$ at round $t$. \\
	$C_{v,i}(\tau)$ & : & $C_{v,i}(\tau)=\sum_{t \in \cT(\tau)}C_{v,i}(t)$ is the corruption added on arm $i$ in agent $v$ in epoch $\tau$. \\
	$C(\tau)$ & : & $C(\tau)=\max_{i \in [K]}\sum_{v=1}^V C_{v,i}(\tau)$ is the corruption level in epoch $\tau$. \\
	$\hat{C}_{v,i}(t)$ & : & $\hat{C}_{v,i}(t)=C_{v,i}(t)y_{v,i}(t)$ is the actual corruption added on arm $i$ in agent $v$ at round $t$. \\
	$\hat{C}_{v,i}(\tau)$ & : & $\hat{C}_{v,i}(\tau)=\sum_{t \in \cT(\tau)}\hat{C}_{v,i}(t)$ is the actual corruption added on arm $i$ in agent $v$ in epoch $\tau$. \\
	$\mathcal{F}_t$ &:& The smallest $\sigma$-algebra containing all information up to round $t$. \\
\end{tabular}
\end{center}

Note that in the above definitions, if agent $v$ is not allocated with arm $i$, then, let associated random variables be zero. For example, we set $\hat{C}_{v,i}(\tau)=0$ and $C_{v,i}(\tau)=0$, if agent $v$ does not have arm $i$.

\section{Lower Bound: proof of Theorem \ref{theorem2}} \label{appendixLB}
To lower-bound the pseudo-regret $R_T$, we start our proof from another regret notation $R'_T$, defined as follows
\begin{equation*}
R'_T  =   \max_{i \in [K]} \sum_{t=1}^T \sum_{v=1}^V \left(r_{v,i}(t)- r_{i_v(t)}(t) \right).
\end{equation*}

The proof is divided into three parts by lower-bounding (i)  $R'_T$, (ii) $\mathbb{E}[R'_T]$, and (iii) $\mathbb{E}[R_T]$ where $R_T$ is defined in \eqref{reg}. In the following, we use $\mathbb{P}_{sto}[\cdot]$ and $\mathbb{E}_{sto}[\cdot]$ to denote the probability and expectation in stochastic setting, respectively. Similarly, the probability and expectation in corrupted setting is denoted by $\mathbb{P}_{cor}[\cdot]$ and $\mathbb{E}_{cor}[\cdot]$, respectively.

\subsection{Preliminaries: two-armed setting and adversary}
\textbf{Setting:} consider a two-armed bandit instance where $V$ agents interact with arm $1$ and arm $2$. The arm $1$ is with Bernoulli reward $\mu_1=1/2-\Delta$ and the arm $2$ is with a constant reward as $r_2=\mu_2=1/2$. We divide the rounds into some intervals of
increasing length $3^\ell T^{\alpha}$  for $\ell=1,2,...L$. For simplicity, we assume $3^\ell T^{\alpha}$ is a integer (If $3^\ell T^{\alpha}$ is not a integer, the length can be modified as $3^\ell  \lfloor T^{\alpha} \rfloor$). Note that interval $L$ might be incomplete as the maximum round is $T$. Therefore, $L \geq \frac{1-\alpha}{\log 3} \log T$.

Let $\tilde{N}_{1}(\ell^*)$ be the number of times that arm $1$ is pulled in the interval $\ell^*$ across all agents.
For any algorithm with pseudo-regret $O(\log (VT) /\Delta)$, there is an interval $\ell^*<L$ such that $\mathbb{E}_{sto}\left[\tilde{N}_{1}(\ell^*) \right] \leq Y$ where for the constant $B_0>0$, $Y$ is given by
\begin{equation*}
Y=\frac{B_0}{\left(1-\alpha \right)\Delta^2}.
\end{equation*}

\textbf{Adversary:} We create an adversary who corrupts the Bernoulli distribution of arm $1$ by setting $\mu_1=1/2+\Delta$, and does nothing for arm $2$. Before interval $\ell^*$, the adversary does not inject any corruption, but starts to inject the corruption in interval $\ell^*$ and beyond. Such a corruption strategy is applied for all agents.
Let $t_{\ell^*} $ be the round at the end of interval $\ell^*$. Let define events $\cE_1$, $\cE_2$ and $\cE_3$ as
\begin{equation*}
\begin{split}
& \cE_1 = \left\{\tilde{N}_{1}(\ell^*) \leq 4 Y \right\}, \\
& \cE_2 = \left\{\sum_{t=1}^{T}\sum_{v=1}^V r_{i_v(t)}(t)<\frac{1}{2}VT+  \left(T-t_{\ell^*} \right)V\Delta +4YV\Delta+ \sqrt{2VT\log (VT)} \right\}, \\
& \cE_3 = \left\{\sum_{t=1}^{T}\sum_{v=1}^V  r_{v,1}(t) \geq \frac{1}{2}VT +3 \Delta  VT^{\alpha}  +\Delta V \left(T-t_{\ell^*} \right) - \sqrt{2VT\log (VT)} \right\}, \\
\end{split}
\end{equation*}
where $\cE_1^c$, $\cE_2^c$, and $\cE_3^c$ are the complementary events, respectively.

\subsection{Lower-bounding $R'_T$} \label{beforemodify}
The proof of $R'_T$ adapts some basic techniques from the single-agent bandit problem, e.g., \cite{BestofBoth3}. The proof is divided into the following steps.

\textbf{Step 1: analyze $\cE_1$.} By Lemma 12 in \cite{BestofBoth3}, we have 
\begin{equation*} 
\mathbb{P}_{cor}\left( \cE_1 \right) \geq \frac{1}{16} \exp \left( -64\Delta^2 Y \right).
\end{equation*}

For simplicity, we dub $p_1=\frac{1}{16} \exp \left( -64\Delta^2 Y \right)$.

\textbf{Step 2: analyze $\cE_2$.} From the construction of the corruption, the following holds under event $\cE_1$.
\begin{equation}  \label{LBtool1}
\begin{split}
\sum_{t=1}^{T} \sum_{v=1}^V \mathbb{E}_{cor}[r_{i_v(t)}(t)|\mathcal{F}_{t-1}] \leq & \frac{1}{2} \left(V t_{\ell^*} -\tilde{N}_{1}(\ell^*) \right)  +\tilde{N}_{1}(\ell^*) \left(\frac{1}{2}+\Delta \right) +V \left(T-t_{\ell^*} \right) \left(\frac{1}{2}+\Delta \right) \\
\leq & \frac{1}{2}VT+  \left(T-t_{\ell^*} \right)V\Delta +4YV\Delta .
\end{split}
\end{equation}

Now, we construct a martingale difference sequence $\{D_{i}(t)\}_{t=0}^{\infty}$ where $D_{i}(t)=\sum_{v=1}^V ( r_{i_v(t)}(t)-\mathbb{E}[r_{i_v(t)}(t) |\mathcal{F}_{t-1}]) $. By Hoeffding-Azuma’s inequality and union bound with $ \cE_1 $, we have
\begin{equation*}
\mathbb{P}_{cor}\left(\cE_2^c \right) \leq  1-\frac{1}{16} \exp \left( -64\Delta^2 Y \right)+ \frac{1}{(VT)^2}.
\end{equation*}

\textbf{Step 3: analyze $\cE_3$.} Define $L(\ell^*)$ as the total number of rounds in interval $\ell^*$ and then we get for arm $1$ that
\begin{equation}  \label{LBtool2}
\begin{split}
\sum_{t=1}^{T} \sum_{v=1}^V \mathbb{E}_{cor}[r_{v,1}(t)] 
= &  \frac{1}{2}VT + \Delta V\left(2L(\ell^*) -t_{\ell^*} \right)  +\Delta V \left(T-t_{\ell^*} \right)\\
\geq &  \frac{1}{2}VT +3 \Delta  VT^{\alpha}  +\Delta V \left(T-t_{\ell^*} \right),
\end{split}
\end{equation}
where $2L(\ell^*) -t_{\ell^*}$ is bounded by
\begin{equation*}
2L(\ell^*) -t_{\ell^*}=L(\ell^*)- \sum_{\ell=1}^{\ell^*-1} L(\ell) \geq 3T^{\alpha}.
\end{equation*}

By Hoeffding-Azuma’s inequality, we have for arm $1$ that
\begin{equation*}
\mathbb{P}_{cor}\left(\cE_3^c \right) \leq  \frac{1}{(VT)^2}.
\end{equation*}

\textbf{Step 4: arm $1$ contributes to more total rewards than arm $2$.} Since the adversary uses the same corruption strategy across all agents, we only focus our analysis of Step 1 on a single agent $v$.
According to our	 construction, the total rewards of arm $2$ for agent $v$ are $\mathbb{E}_{cor}[ \sum_t r_{v,2}(t)]=\frac{1}{2}T$, and the total rewards of arm $1$ for agent $v$ are presented in \eqref{LBtool2}. Let $\Delta_v(t) =\sum_t r_{v,2}(t)-\sum_t r_{v,1}(t)$.
As all rewards are independent, we use Hoeffding inequality to get
\begin{equation*}
\begin{split}
\mathbb{P}_{cor}\left[\sum_{t=1}^T r_{v,2}(t) >\sum_{t=1}^T r_{v,1}(t) \right] \leq &  \mathbb{P}_{cor}\left[\left|\sum_{t=1}^T \Delta_v(t) -\mathbb{E}_{cor} [\Delta_v(t)] \right|>3 \Delta  VT^{\alpha}  +\Delta V \left(T-t_{\ell^*} \right) \right]  \\
\leq & \exp \left(-\frac{\left( 3   VT^{\alpha}  + V \left(T-t_{\ell^*} \right) \right)^2}{T} \right)=p_2.
\end{split}
\end{equation*}

Thus, given any agent $v$, with probability at least $1-p_2$, the total rewards of arm $2$ is less than the total rewards of arm $1$ over time horizon $T$. Then, by a union bound over all $V$ agents, $R'_T$ with probability at least $1-Vp_2$, is 
\begin{equation}\label{acreg}
R'_T  =   \max_{i \in [K]} \sum_{t=1}^T \sum_{v=1}^V  \left( r_{v,i}(t)-r_{i_v(t)}(t)\right) =\sum_{t=1}^T \sum_{v=1}^V \left(r_{v,1}(t)- r_{i_v(t)}(t) \right).
\end{equation}

\textbf{Step 5: put together.} Take a union bound, and then use the fact that $C=VT^{\alpha}$, \eqref{LBtool1}, \eqref{LBtool2}, and \eqref{acreg} to get the probability at least $p_1- \frac{2}{(VT)^2}-Vp_2$ (Note that for a sufficienlty large $T$, one can have that $p_1>Vp_2$),
\begin{equation}  \label{LBresultTmp1}
\sum_{t=1}^T \sum_{v=1}^V \left(r_{v,1}(t)- r_{i_v(t)}(t) \right) \geq \Delta  C-4Y \Delta-2\sqrt{2VT\log(VT)}.
\end{equation}
where $4\Delta Y= \frac{4B_0}{\left(1-\alpha \right)\Delta} \leq \sqrt{2VT\log(VT)}$ with a large $T$. Combing \eqref{LBresultTmp1}, we have the following with probability at least $p_1- \frac{2}{(VT)^2}-Vp_2$
\begin{equation}  \label{LBresultTmp2}
R'_T=\sum_{t=1}^T \sum_{v=1}^V \left(r_{v,1}(t)- r_{i_v(t)}(t) \right) \geq \Delta  C-3\sqrt{2VT\log(VT)}.
\end{equation}

\subsection{Lower-bounding $\mathbb{E}[R'_T]$} \label{LBRPT}
\textbf{Modification of adversary.} To lower bound $\mathbb{E}[R'_T]$, we slightly modify the corrupton strategy. Specifically, if
there is a round $t \leq t_{\ell^*}$ such that the number of pulls of arm $1$ in interval $\ell^*$ exceeds $4Y$, then for all remaining rounds, the adversary does not corrupt any more. 

After modifying the corrupton strategy, one can observe that under $\tilde{N}_{1}(\ell^*) \leq 4 Y$, the adversary will always inject corruption in interval $\ell^*$ and beyond. Therefore, from the construction of interval, we have that under $\tilde{N}_{1}(\ell^*) \leq 4 Y$, with probability at least $1-Vp_2$, $R'_T=\sum_{t=1}^T \sum_{v=1}^V \left(r_{v,1}(t)- r_{i_v(t)}(t) \right)$.
Thus, the analysis in Appendix \ref{beforemodify} is also applicable for regret analysis after modifying the adversary, under $\tilde{N}_{1}(\ell^*) \leq 4 Y$.
Under $\tilde{N}_{1}(\ell^*) > 4 Y$, it is known that there should exist a round such that the adversary stops
injecting corruption. In this case, with a high probability, $R'_T=\sum_{t=1}^T \sum_{v=1}^V \left(r_{v,2}(t)- r_{i_v(t)}(t) \right)$ as the adversary stops injecting corruption in interval $\ell^*$, and 	arm $2$ will yield more rewards than that of arm $1$.

For notational simplicity, we dub $p_3=p_1- \frac{2}{(VT)^2}-V p_2$. In the following, all expectations are taken for the corrupted setting, and hence we use $\mathbb{E}[\cdot]$ to avoid clutter. Then, we write $\mathbb{E}[R'_T]$ as
\begin{equation*}  
\mathbb{E}[R'_T]= \mathbb{E} \left[R'_T  \big |  \tilde{N}_{1}(\ell^*) \leq 4 Y  \right]  \mathbb{P} \left(\tilde{N}_{1}(\ell^*) \leq 4 Y \right)+ \mathbb{E} \left[R'_T  \big |  \tilde{N}_{1}(\ell^*) > 4 Y  \right]  \mathbb{P} \left(\tilde{N}_{1}(\ell^*) > 4 Y \right),
\end{equation*}
and then use the result of \eqref{LBresultTmp2} to bound the first term as
\begin{equation*}  
\begin{split}
& \mathbb{E} \left[R'_T  \big |  \tilde{N}_{1}(\ell^*) \leq 4 Y  \right]  \mathbb{P} \left(\tilde{N}_{1}(\ell^*) \leq 4 Y \right)\\
\geq &\left( \Delta  \mathbb{E}[C]-3\sqrt{2VT\log(VT)}\right)  \mathbb{P} \left( \tilde{N}_{1}(\ell^*) \leq 4 Y , R'_T \geq \Delta C-3\sqrt{2VT\log(VT)} \right) \\
&-VT \cdot \mathbb{P} \left(\tilde{N}_{1}(\ell^*) \leq 4 Y , R'_T < \Delta  C-3\sqrt{2VT\log(VT)} \right) \\
\geq &p_3 \left( \Delta \mathbb{E}[ C]-3\sqrt{2VT\log(VT)}\right)  
-\frac{2}{VT} ,
\end{split}
\end{equation*}
where the last inequality is due to the following steps
\begin{itemize}
\item First, the following holds
\begin{equation*}
\begin{split}
& \mathbb{P} \left(\tilde{N}_{1}(\ell^*) \leq 4 Y , R'_T < \Delta  C-3\sqrt{2VT\log(VT)} \right) \\
 = & \mathbb{P} \left( R'_T < \Delta  C-3\sqrt{2VT\log(VT)} \Big | \tilde{N}_{1}(\ell^*) \leq 4 Y  \right) \mathbb{P}  \left(  \tilde{N}_{1}(\ell^*) \leq 4 Y  \right) .
\end{split}
\end{equation*}
\item Second, from the analysis in Appendix \ref{beforemodify}, we have known that $ \mathbb{P} (\cE_2 \mid  \tilde{N}_{1}(\ell^*) \leq 4 Y ) \leq 1/(VT)^2 $ and $ \mathbb{P} (\cE_3 \mid  \tilde{N}_{1}(\ell^*) \leq 4 Y ) \leq 1/(VT)^2 $.
\item Finally, by a union bound, the following holds.
\begin{equation*}
\begin{split}
 \mathbb{P} \left( \sum_{t=1}^T \sum_{v=1}^V \left(r_{v,1}(t)- r_{i_v(t)}(t) \right)  < \Delta  C-3\sqrt{2VT\log(VT)} \bigg|  \tilde{N}_{1}(\ell^*) \leq 4 Y  \right)  \leq \frac{2}{(VT)^2}.
\end{split}
\end{equation*}
\end{itemize}

By a similar method, we have that 
\begin{equation*}
\begin{split}
& \mathbb{E} \left[R'_T  \Big |  \tilde{N}_{1}(\ell^*) > 4 Y  \right]  \mathbb{P} \left(\tilde{N}_{1}(\ell^*) > 4 Y \right) \\
=&\mathbb{E} \left[\sum_{t=1}^T \sum_{v=1}^V r_{v,2}(t)- r_{i_v(t)}(t)   \bigg|   \tilde{N}_{1}(\ell^*) > 4 Y  \right]  \mathbb{P} \left(\tilde{N}_{1}(\ell^*) > 4 Y \right) \\
&+\mathbb{E} \left[\sum_{t=1}^T \sum_{v=1}^V r_{v,1}(t)- r_{i_v(t)}(t)   \bigg|   \tilde{N}_{1}(\ell^*) > 4 Y  \right]  \mathbb{P} \left(\tilde{N}_{1}(\ell^*) > 4 Y \right) \\
\geq &- 3\sqrt{VT\log (VT)}-\frac{3}{VT},
\end{split}
\end{equation*}
where the last inequality bounds $\sum_{t=1}^T \sum_{v=1}^V r_{v,1}(t)- r_{i_v(t)}(t) \geq -VT $, and the probability $\mathbb{P}(R'_T=\sum_{t=1}^T \sum_{v=1}^V r_{v,1}(t)- r_{i_v(t)}(t) | \tilde{N}_{1}(\ell^*) > 4 Y )$ is at most $1/(VT)^2$ since the corruption injected by the adversary is at most $4Y$ and the environment is close to being stochastic.

Combing the above, we get for some constant $0<B_1 <p_3 $,
\begin{equation}  \label{LBresult1}
\mathbb{E}[R'_T] \geq  B_1 \left(  \mathbb{E}[C]-\sqrt{VT\log(VT)} \right).
\end{equation}

\subsection{Lower-bounding $\mathbb{E}[R_T]$}
In the two-armed setting, we have that $\mathbb{E}[R_T]=\mathbb{E}[ \sum_{t=1}^T \sum_{v=1}^V (r^S_{v,2}(t)- r^S_{i_v(t)}(t) )]$. Thus, to lower-bound $\mathbb{E}[R_T]$, we need to connect $\sum_{t=1}^T \sum_{v=1}^V (r^S_{v,2}(t)- r^S_{i_v(t)}(t) )$ and $R'_T$.
From the analysis from Appendix \ref{LBRPT}, we know that with probability at least $1-Vp_2$, $R'_T=\sum_{t=1}^T \sum_{v=1}^V (r_{v,1}(t)- r_{i_v(t)}(t) )$ under $\tilde{N}_{1}(\ell^*) \leq 4 Y$, whereas with high probability, $R'_T=\sum_{t=1}^T \sum_{v=1}^V (r_{v,2}(t)- r_{i_v(t)}(t) )$ under $\tilde{N}_{1}(\ell^*) > 4 Y$. We first decompose $\sum_{t=1}^T \sum_{v=1}^V (r_{v,1}(t)- r_{i_v(t)}(t) )$ as
\begin{equation*}
\begin{split}
& \sum_{t=1}^T \sum_{v=1}^V \left(r_{v,1}(t)- r_{i_v(t)}(t) \right)  \\
= & \sum_{t=1}^T \sum_{v=1}^V \left(r^S_{v,1}(t)- r^S_{i_v(t)}(t) \right)
+ \sum_{t=1}^T \sum_{v=1}^V \left(r^S_{i_v(t)}(t)- r_{i_v(t)}(t) \right) 
+ \sum_{t=1}^T \sum_{v=1}^V \left(r_{v,1}(t)- r^S_{v,1}(t) \right)  \\
= & \sum_{t=1}^T \sum_{v=1}^V \left(r^S_{v,1}(t)- r^S_{i_v(t)}(t) \right)
+ \sum_{t=1}^T \sum_{v=1}^V \left(r^S_{i_v(t)}(t)- r_{i_v(t)}(t) \right) 
+ C \\
= & \sum_{t=1}^T \sum_{v=1}^V \left(r^S_{v,2}(t)- r^S_{i_v(t)}(t) \right) 
+\sum_{t=1}^T \sum_{v=1}^V \left(r^S_{v,1}(t)- r^S_{v,2}(t) \right)
+ \sum_{t=1}^T \sum_{v=1}^V \left(r^S_{i_v(t)}(t)- r_{i_v(t)}(t) \right) 
+ C .
\end{split}
\end{equation*}

Then, we decompose $\sum_{t=1}^T \sum_{v=1}^V (r_{v,2}(t)- r_{i_v(t)}(t) )$ as
\begin{equation*}
\begin{split}
& \sum_{t=1}^T \sum_{v=1}^V \left(r_{v,2}(t)- r_{i_v(t)}(t) \right)  \\
= & \sum_{t=1}^T \sum_{v=1}^V \left(r^S_{v,2}(t)- r^S_{i_v(t)}(t) \right)
+ \sum_{t=1}^T \sum_{v=1}^V \left(r^S_{i_v(t)}(t)- r_{i_v(t)}(t) \right) 
+ \sum_{t=1}^T \sum_{v=1}^V \left(r_{v,2}(t)- r^S_{v,2}(t) \right)  \\
= & \sum_{t=1}^T \sum_{v=1}^V \left(r^S_{v,2}(t)- r^S_{i_v(t)}(t) \right)
+ \sum_{t=1}^T \sum_{v=1}^V \left(r^S_{i_v(t)}(t)- r_{i_v(t)}(t) \right) .
\end{split}
\end{equation*}

Take expectation over $R'_T  $ under $\tilde{N}_{1}(\ell^*) \leq 4 Y $ to have
\begin{equation}\label{regunderleq}
\begin{split}
& \mathbb{E}[R'_T |\tilde{N}_{1}(\ell^*) \leq 4 Y ]\\
\leq& \mathbb{E}\left[ \sum_{t=1}^T \sum_{v=1}^V \left(r^S_{v,2}(t)- r^S_{i_v(t)}(t) \right) +C   \Big|\tilde{N}_{1}(\ell^*) \leq 4 Y \right]
-\Delta VT \\
&+ \mathbb{E}\left[ \sum_{t=1}^T \sum_{v=1}^V \left(r^S_{i_v(t)}(t)- r_{i_v(t)}(t) \right) \bigg|  \tilde{N}_{1}(\ell^*) \leq 4 Y \right] +1
,
\end{split}
\end{equation}
where the inequality is due to the following reasons. First, the generation of $ r^S_{v,1}(t)- r^S_{v,2}(t) $ for each round $t$ is independent of the history. Then, the summation over $\mathbb{E}[ r^S_{v,1}(t)- r^S_{v,2}(t)] $ can be bounded by $-\Delta VT$. Second, for a sufficiently large $T$, one can simply bound $Vp_2 \leq 1/VT$, so that $R'_T =\sum_{t=1}^T \sum_{v=1}^V (r_{v,1}(t)- r_{i_v(t)}(t) ) $ with probability at least $1-1/VT$, and the expected regret with the remaining probability of $R'_T =\sum_{t=1}^T \sum_{v=1}^V (r_{v,2}(t)- r_{i_v(t)}(t) ) $ can be trivially bounded by $1$.

Again, take expectation over $R'_T  $ under $\tilde{N}_{1}(\ell^*) > 4 Y $ to have
\begin{equation}\label{regundergeq}
\begin{split}
& \mathbb{E}[R'_T |\tilde{N}_{1}(\ell^*) > 4 Y ] \\
\leq & \mathbb{E}\left[ \sum_{t=1}^T \sum_{v=1}^V \left(r^S_{v,2}(t)- r^S_{i_v(t)}(t) \right) \bigg| \tilde{N}_{1}(\ell^*) > 4 Y  \right]\\
&+ \mathbb{E}\left[ \sum_{t=1}^T \sum_{v=1}^V \left(r^S_{i_v(t)}(t)- r_{i_v(t)}(t) \right)\bigg|  \tilde{N}_{1}(\ell^*) > 4 Y   \right]  +1,
\end{split}
\end{equation}
where under this event, the probability of $R'_T =\sum_{t=1}^T \sum_{v=1}^V (r_{v,1}(t)- r_{i_v(t)}(t) ) $ is at most $1/VT$ for a sufficiently large $T$ and thus the expected regret can be also trivially bounded by $1$.

Then, we bound \eqref{regunderleq} and \eqref{regundergeq}, respectively. 

\subsection*{Under $\tilde{N}_{1}(\ell^*) > 4 Y $.}

In this case, the adversary stops injecting the corruption at a certain round $t$ such that $t \leq t_{\ell^*}$. According to the corruption strategy, it is known that $\mathbb{E}[r^S_{i_v(t)}(t)- r_{i_v(t)}(t)]=-2\Delta$ for $i_v(t)=1$, and $r^S_{i_v(t)}(t)- r_{i_v(t)}(t) =0$ for $i_v(t)=2$. Under $\tilde{N}_{1}(\ell^*) > 4 Y $, the corruption in interval $\ell^*$ ends upon the number of pulls of arm $1$ exceeds $4 Y $. From 
the analysis of Step 1 in Appendix \ref{beforemodify}, we know that the algorithm will not detect the corruption on arm $1$ in interval $\ell^*$, and thus the algorithm behaves as the environment was almost stochastic. Such a corruption is $O(1/\Delta)$, which implies that the expected corruption level is
$\mathbb{E}[C | \tilde{N}_{1}(\ell^*) > 4 Y] = B_2/\Delta 
$ for a suitable constant $B_2>0$.
\begin{equation}\label{LBtool4}
\begin{split}
 & \mathbb{E}\left[ \sum_{t=1}^T \sum_{v=1}^V \left(r^S_{i_v(t)}(t)- r_{i_v(t)}(t) \right)  \bigg | \tilde{N}_{1}(\ell^*) > 4 Y \right]\\
  = &  
\mathbb{E}\left[ \sum_{t=1}^T \sum_{v=1}^V \left(r^S_{v,1}(t)- r_{v,1}(t) \right) \mathbb{I}\{ i_v(t)=1\}  \bigg | \tilde{N}_{1}(\ell^*) > 4 Y \right]
 \\
=& -\frac{B_2}{\Delta }.
\end{split}
\end{equation}

We use \eqref{regundergeq}, \eqref{LBtool4}, and $\mathbb{E}[C | \tilde{N}_{1}(\ell^*) > 4 Y] = B_2 /\Delta 
$ to get
\begin{equation}\label{LBresult2}
\begin{split}
 & \mathbb{E}\left[ \sum_{t=1}^T \sum_{v=1}^V \left(r^S_{v,2}(t)- r^S_{i_v(t)}(t) \right) \Big | \tilde{N}_{1}(\ell^*) > 4 Y \right]  \\
  \geq & \mathbb{E}[R'_T  | \tilde{N}_{1}(\ell^*) > 4 Y]
    + \frac{B_2 \log(VT)}{\Delta }-1 \\
  \geq  & \mathbb{E}[R'_T | \tilde{N}_{1}(\ell^*) > 4 Y]+\mathbb{E}[C  | \tilde{N}_{1}(\ell^*) > 4 Y] -1 .
\end{split}
\end{equation}

\subsection*{Under $\tilde{N}_{1}(\ell^*) \leq 4 Y $.}
Next, we connect $ \mathbb{E}[ \sum_{t=1}^T \sum_{v=1}^V (r^S_{v,2}(t)- r^S_{i_v(t)}(t) ) +C  |\tilde{N}_{1}(\ell^*) \leq 4 Y ]$ and $\mathbb{E}[R'_T |\tilde{N}_{1}(\ell^*) \leq 4 Y ]$ by considering the following two cases.

\textbf{Case 1: $\ell^*=1$.} This case implies that the adversary injects the corruption from the beginning to the end. In this case, the reward of arm $1$ is always sampled from a fixed Bernoulli distribution with mean $1/2+\Delta$, while arm $2$ is with a constant reward $1/2$. Such a setting can be simulated as a ``mirror stochastic'' setting (compared with $\mu_1=1/2-\Delta$ and $\mu_2=1/2$). In this setting, the suboptimal arm (in this mirror stochastic setting) is arm $2$. 
Thus, for any algorithm that enjoys a $O(\log VT/\Delta)$ pseudo-regret, the expected number of pull of the suboptimal arm is $B_3 \log(VT)/\Delta^2$ for a suitable constant $B_3>0$.

Define $\tilde{N}_2$ as the number of pulls of arm $2$ over all rounds and agents. Then, we have
\begin{equation}\label{LBtool41}
\begin{split}
&  \mathbb{E}\left[ \sum_{t=1}^T \sum_{v=1}^V \left(r^S_{i_v(t)}(t)- r_{i_v(t)}(t) \right)  \bigg | \tilde{N}_{1}(\ell^*) \leq 4 Y \right]\\
 = &  
\mathbb{E}\left[ \sum_{t=1}^T \sum_{v=1}^V \left(r^S_{1}(t)- r_{1}(t) \right) \mathbb{I}\{ i_v(t)=1\}  \bigg | \tilde{N}_{1}(\ell^*) \leq 4 Y \right]
 \\
=&
-2\Delta \left(  VT -  \mathbb{E}[\tilde{N}_2| \tilde{N}_{1}(\ell^*) \leq 4 Y ] \right)
 \\
=&-2\Delta  VT+\frac{B_3 \log(VT)}{\Delta }\\
\leq &-\frac{3}{2}\Delta  VT,
\end{split}
\end{equation}
where the last inequality follows that $B_3 \log(VT)/\Delta  \leq \Delta  VT/2$ for a sufficiently large $T$.

Under $\tilde{N}_{1}(\ell^*) \leq 4 Y $, the corruption takes place in all rounds, the expected corruption level is 
$\mathbb{E}[C | \tilde{N}_{1}(\ell^*) \leq 4 Y] \leq 2\Delta  VT 
$. We use \eqref{regunderleq}, \eqref{LBtool41}, and $\mathbb{E}[C | \tilde{N}_{1}(\ell^*) \leq 4 Y] \leq 2\Delta  VT 
$ to get
\begin{equation}\label{LBresult1}
\begin{split}
 & \mathbb{E}\left[ \sum_{t=1}^T \sum_{v=1}^V \left(r^S_{v,2}(t)- r^S_{i_v(t)}(t) \right) \Big | \tilde{N}_{1}(\ell^*) \leq 4 Y \right]  \\ 
 \geq & \mathbb{E}[R'_T  | \tilde{N}_{1}(\ell^*) \leq 4 Y]
+ \frac{5}{2}\Delta VT
- \mathbb{E}[C  | \tilde{N}_{1}(\ell^*) \leq 4 Y]-1 \\
  \geq  & \mathbb{E}[R'_T | \tilde{N}_{1}(\ell^*) \leq 4 Y]+\frac{1}{4}\mathbb{E}[C  | \tilde{N}_{1}(\ell^*) \leq 4 Y]  -1
.
\end{split}
\end{equation}

\textbf{Case 2: $\ell^* \geq 2$.} According to the corruption strategy (if corruption occurs, then, the adversary always drags down the reward of arm $1$), $  \sum_{t=1}^T \sum_{v=1}^V \left(r^S_{i_v(t)}(t)- r_{i_v(t)}(t) \right) \leq 0$. Under $\tilde{N}_{1}(\ell^*) \leq 4 Y$, \eqref{regunderleq} is upper-bounded as
\begin{equation}\label{LBtool5}
\begin{split}
& \mathbb{E}[R'_T| \tilde{N}_{1}(\ell^*) \leq 4 Y ] \\ \leq &  \mathbb{E}\left[ \sum_{t=1}^T \sum_{v=1}^V \left(r^S_{v,2}(t)- r^S_{i_v(t)}(t) \right) \Big | \tilde{N}_{1}(\ell^*) \leq 4 Y \right] 
-\Delta VT
+ \mathbb{E}[C | \tilde{N}_{1}(\ell^*) \leq 4 Y] .
\end{split}
\end{equation}

As $\ell^*\geq 2$, under $\tilde{N}_{1}(\ell^*) \leq 4 Y $, the corruption is upper-bounded as
\begin{equation*}
\begin{split}
C \leq 2\Delta  V\left(T-\sum_{\ell=1}^{\ell^*-1} L_i  \right) =2\Delta   V\left(T-T^{\alpha}  \frac{3^{\ell^*}-1}{2}  \right) \leq  2\Delta   \left(VT-4C  \right),
\end{split}
\end{equation*}
which immediately leads to
\begin{equation} \label{LBDeltaVT}
\Delta V T \geq \left(\frac{1}{2}+4 \Delta  \right)\mathbb{E}[C | \tilde{N}_{1}(\ell^*) \leq 4 Y ].
\end{equation}

Using \eqref{LBtool5}, \eqref{LBDeltaVT} and \eqref{LBresult1}, we lower bound $\mathbb{E}[R_T] $ as
\begin{equation}\label{LBresult2}
\begin{split}
& \mathbb{E}\left[ \sum_{t=1}^T \sum_{v=1}^V \left(r^S_{v,2}(t)- r^S_{i_v(t)}(t) \right) \Big | \tilde{N}_{1}(\ell^*) \leq 4 Y \right] \\
 \geq  & \mathbb{E}[R'_T | \tilde{N}_{1}(\ell^*) \leq 4 Y ]
+\left(4\Delta -\frac{1}{2} \right)\mathbb{E}[C | \tilde{N}_{1}(\ell^*) \leq 4 Y]  .
\end{split}
\end{equation}

\textbf{Put two cases together.} Since $\Delta \in (1/4,1/2)$, $\left(4\Delta -\frac{1}{2} \right)$ is positive for $\Delta \in (1/4,1/2)$. Combing the above analysis, we have that
\begin{equation*}
\begin{split}
\mathbb{E}[R_T]= &\mathbb{E}\left[ \sum_{t=1}^T \sum_{v=1}^V \left(r^S_{v,2}(t)- r^S_{i_v(t)}(t) \right) \Big | \tilde{N}_{1}(\ell^*) \leq 4 Y \right]  \mathbb{P} \left(\tilde{N}_{1}(\ell^*) \leq 4 Y \right)\\
&+\mathbb{E}\left[ \sum_{t=1}^T \sum_{v=1}^V \left(r^S_{v,2}(t)- r^S_{i_v(t)}(t) \right) \Big | \tilde{N}_{1}(\ell^*) > 4 Y \right] \mathbb{P} \left(\tilde{N}_{1}(\ell^*) > 4 Y \right)  \\
= &\Omega \left( \mathbb{E} \left[R'_T   \right] + \mathbb{E}[C ]  \right),
\end{split}
\end{equation*}
where $\mathbb{P} (\tilde{N}_{1}(\ell^*) \leq 4 Y )=p_1$ (see Appendix \ref{beforemodify}), and thus we get the desired result.

\section{Proof of Lemma \ref{Tmpconcentration2All}}\label{techlemmas}

\begin{lemma}  \label{boundedpull}
For any fixed $v,i, \tau$ the following holds
\begin{equation} \label{concentration2Pull}
\mathbb{P}\left( \tilde{N}_{v,i}(\tau)\leq 3N_{v,i}(\tau) \right) \geq 1-\frac{\delta}{2VK\log_4T}.
\end{equation}
\end{lemma}

\begin{proof}
Define $
D_{v,i}(t)=y_{v,i}(t)-p_{v,i}(\tau)
$ (the definition of $y_{v,i}(t)$ is given in Appendix \ref{notationtable}).
Then, $\{D_{v,i}(t)\}_{t=0}^{\infty}$ is a martingale difference sequence with
respect to the filtration $\{\mathcal{F}_t\}_{t=0}^{\infty}$. Then, $\mathbb{E}\left[D_{v,i}(t)|\mathcal{F}_{t-1}\right]=0$, and then the martingale variance is equal to $\mathbb{E}\left[D^2_{v,i}(t)|\mathcal{F}_{t-1}\right]$, which is bounded as
\begin{equation*}
V=\sum_{t \in \mathcal{T}(\tau)} \mathbb{E}\left[D_{v,i}^2(t)|\mathcal{F}_{t-1}\right] \leq \sum_{t \in \mathcal{T}(\tau)}  \mathrm{Var}\left(y_{v,i}(t)\right) \leq p_{v,i}(\tau) N(\tau).
\end{equation*}

As $|D_{v,i}(t)| \leq 1$, applying Freedman’s inequality (see Lemma \ref{Freedm}) gives with probability at least $1-\delta'$,
\begin{equation*}
\sum_{t \in \mathcal{T}(\tau)}  D_{v,i}(t) \leq V+\log(1/\delta') = p_{v,i}(\tau) N(\tau)+\log(1/\delta') .
\end{equation*}

Then, by choosing $\delta'=\delta/(2VK\log_4T)$, we get 
\begin{equation*}
\begin{split}
\tilde{N}_{v,i}(\tau) = \sum_{t \in \mathcal{T}(\tau)} \left( D_{v,i}(t) +p_{v,i}(\tau) \right) \leq & 2p_{v,i}(\tau) N(\tau) +\log((2VK\log_4T)/\delta) \leq  3p_{v,i}(\tau) N(\tau),
\end{split}
\end{equation*}
where the last inequality holds as $p_{v,i}(\tau) N(\tau) >  \log((8VK\log_4T)/\delta) > \log((2VK\log_4T)/\delta)  $. We complete the proof as $p_{v,i}(\tau) N(\tau)=N_{v,i}(\tau)$.

\end{proof}

\begin{lemma}  \label{boundedcorruption}
For any fixed $v$, $i$, $\tau$, and any corruption level $C$, the following holds
\begin{equation*}
\mathbb{P}\left( \hat{C}_{v,i}(\tau)  \leq 2p_{v,i}(\tau) C(\tau)  +\log((8VK\log_4T)/\delta) \right) \geq 1-\frac{\delta}{8VK\log_4T}.
\end{equation*}
\end{lemma}

\begin{proof}
We define
$
D_{v,i}(t)=C_{v,i}(t)(y_{v,i}(t)-p_{v,i}(\tau))
$ where the definition of $y_{v,i}(t)$ is given in Appendix \ref{notationtable}.
Then, $\{D_{v,i}(t)\}_{t=0}^{\infty}$ is a martingale difference sequence with
respect to the filtration $\{\mathcal{F}_t\}_{t=0}^{\infty}$. We have that $\mathbb{E}\left[D_{v,i}(t)|\mathcal{F}_{t-1}\right]=0$ and $C_{v,i}(t)$ is deterministic given $\mathcal{F}_{t-1}$, i.e., $\mathbb{E}[C_{v,i}(t)|\mathcal{F}_{t-1}]=C_{v,i}(t)$ (as the corruption at round $t$ is a function of the history information). Then, the martingale variance $V$ is equal to $\mathbb{E}\left[D^2_{v,i}(t)|\mathcal{F}_{t-1}\right]$, which is bounded as:
\begin{equation*}
V=\sum_{t \in \mathcal{T}(\tau)} \mathbb{E}\left[D_{v,i}^2(t)|\mathcal{F}_{t-1}\right] \leq \sum_{t \in \mathcal{T}(\tau)} C_{v,i}(t) \mathrm{Var}\left(y_{v,i}(t)\right)\leq p_{v,i}(\tau) C_{v,i}(\tau) \leq p_{v,i}(\tau) C(\tau).
\end{equation*}

As $|D_i(t)| \leq 1$, applying Freedman’s inequality (see Lemma \ref{Freedm}) gives, with probability at least $1-\delta'$,
\begin{equation*}
\sum_{t \in \mathcal{T}(\tau)}  D_{v,i}(t) \leq V+\log(1/\delta') \leq p_{v,i}(\tau) C(\tau) +\log(1/\delta') .
\end{equation*}

Then, by choosing $\delta'=\delta/(8VK\log_4T)$, we get with probability at least $1-\delta/(8VK\log_4T)$,
\begin{equation*}
\begin{split}
\hat{C}_{v,i}(\tau)  = \sum_{t \in \mathcal{T}(\tau)} \left( D_{v,i}(t) +C_{v,i}(t) p_{v,i}(\tau) \right) 
\leq  2p_{v,i}(\tau) C(\tau) +\log((8VK\log_4T)/\delta),
\end{split}
\end{equation*}
which concludes the proof.

\end{proof}

\begin{lemma} \label{lemma1}
For any fixed $v,i,\tau$, the following holds
\begin{equation} \label{concentrationSto}
\mathbb{P} \left(\left| \hat{\mu}_{v,i}^S(\tau) -\mu_i \right| >  \frac{14}{9}\epsilon_{i}(\tau)\right) \leq \frac{\delta}{4VK\log_4T} . 
\end{equation}
\end{lemma}

\begin{proof}
We define
$
D_{v,i}(t)=   r^S_{v,i}(t) y_{v,i} (t)  -\mu_i p_{v,i}(\tau)  
$ where the definition of $y_{v,i}(t)$ is given in Appendix \ref{notationtable}.
Then, $\{D_{v,i}(t)\}_{t=0}^{\infty}$ is a martingale difference sequence with
respect to the filtration $\{\mathcal{F}_t\}_{t=0}^{\infty}$. By Hoeffding-Azuma’s inequality (see Lemma \ref{HAB}), with probability at most $\delta'$,
\begin{equation*}
 \left| \sum_{t \in \mathcal{T}(\tau)} D_{v,i}(t) \right| >  \sqrt{\frac{\log(2/\delta')}{2}\sum_{t \in \mathcal{T}(\tau)} \left|D_{v,i}(t)-D_{v,i}(t-1) \right|} .
\end{equation*}

The left term in the above can be bounded as
\begin{equation*}
\begin{split}
 \sqrt{\frac{\log(2/\delta')}{2}\sum_{t \in \mathcal{T}(\tau)} \left|D_{v,i}(t)-D_{v,i}(t-1) \right|} 
\leq & \sqrt{\frac{\log(2/\delta')}{2}\sum_{t \in \mathcal{T}(\tau)} \max\{y_{v,i} (t-1),y_{v,i} (t)\}}\\
\leq & \sqrt{\frac{\log(2/\delta')}{2}\sum_{t \in \mathcal{T}(\tau)} \left(y_{v,i} (t-1)+ y_{v,i} (t) \right)}\\
\leq & \sqrt{\log(2/\delta') \tilde{N}_{v,i}(\tau) }\\
\leq & \sqrt{3\log(2/\delta') N_{v,i}(\tau) },
\end{split}
\end{equation*}
where the second inequality is due to the fact that $\max \{\alpha,\beta \} \leq \alpha+\beta$ holds for any $\alpha,\beta \geq 0$, and the last inequality uses Lemma \ref{boundedpull}. According to the definition of $D_{v,i}(t)$, we get
\begin{equation*}
\frac{\sum_{t \in \mathcal{T}(\tau)} D_{v,i}(t)}{N_{v,i}(\tau)}=\frac{\sum_{t \in \mathcal{T}(\tau)}  \left( r^S_{v,i}(t) y_{v,i} (t)  -\mu_i p_{v,i}(\tau) \right) }{N_{v,i}(\tau)} =\hat{\mu}_{v,i}^S(\tau) -\mu_i.
\end{equation*}

By choosing $\delta'=\delta/4VK\log_4T$, we have
\begin{equation*}
\mathbb{P} \left(\left| \hat{\mu}_{v,i}^S(\tau) -\mu_i \right| >   \sqrt{\frac{3\log(8VK\log_4T/\delta)}{ N_{v,i}(\tau) }}\right) \leq \frac{\delta}{4VK\log_4T} . 
\end{equation*}

For $i \in \cA_v(\tau)$, we have
\begin{equation*}
\sqrt{\frac{3\log(8VK\log_4T/\delta)}{N_{v,i}(\tau)}}= \sqrt{\frac{\left|\cA_{v}(\tau) \right| \epsilon^2(\tau)}{\widetilde{K} \left(1-\sum_{j \in \cB_v(\tau)} p_{v,j}(\tau) \right)}} \leq  \frac{2\sqrt{3}}{3}\epsilon(\tau)< \frac{14}{9}\epsilon_{i}(\tau),
\end{equation*}
where $\sum_{j \in \cB_v(\tau)} p_{v,j}(\tau) $ is bounded by
\begin{equation} \label{boundpj}
\sum_{j \in \cB_v(\tau)} p_{v,j}(\tau) \leq  \sum_{j \in \cB_v(\tau)} \frac{\epsilon^2(\tau)}{\epsilon^2_{j}(\tau)\widetilde{K}} \leq \frac{1}{4}.
\end{equation}

For $i \in \cB_v(\tau)$, $ \sqrt{3\log(8VK\log_4T/\delta)/N_{v,i}(\tau) }=\epsilon_{i}(\tau) < \frac{14}{9}\epsilon_{i}(\tau) $ also holds. As a result, we complete the proof.
\end{proof}

\begin{lemma} \label{concentration}
For any fixed $v,i,\tau$, the following holds
\begin{equation} \label{concentration2Estima}
\mathbb{P} \left(\left|\hat{\mu}_{v,i}(\tau)-\mu_i \right|   >  \mu_i +2\epsilon_{i}(\tau) + \frac{2 C(\tau)}{N(\tau)}\right)  \leq    \frac{\delta}{2VK\log_4T}.
\end{equation}
\end{lemma}

\begin{proof}
We have
\begin{equation*} 
\begin{split}
\hat{\mu}_{v,i}(\tau) \leq  \frac{\sum_{t \in \mathcal{T}(\tau)} \left(  r^S_{v,i}(t)+\hat{C}_{v,i}(t)\right)}{N_{v,i}(\tau)}  \leq   \hat{\mu}^{S}_{v,i}(\tau)+  \frac{\sum_{t \in \mathcal{T}(\tau)} \hat{C}_{v,i}(t)}{N_{v,i}(\tau)}  .
 \end{split}
\end{equation*}

By Lemma \eqref{boundedcorruption} and the fact $p_{v,i}(\tau)N(\tau)=N_{v,i}(\tau)$, with probability at least $1-\delta/(8VK\log_4 T)$, we have for any $i$ that
\begin{equation*} 
\frac{\sum_{t \in \mathcal{T}(\tau)} \hat{C}_{v,i}(t)}{N_{v,i}(\tau)} \leq     \frac{2p_{v,i}(\tau) C_{v,i}(\tau) +\log((8VK\log_4T)/\delta)}{N_{v,i}(\tau)} 
\leq    \frac{2 C(\tau)}{N(\tau)}+ \frac{\log((8VK\log_4T)/\delta)}{N_{v,i}(\tau)} .
\end{equation*}

Then, for $i \in \cA_v(\tau)$, we have that $\epsilon(\tau)=\epsilon_i(\tau)$, which further gives
\begin{equation*} 
\frac{\log((8VK\log_4T)/\delta)}{N_{v,i}(\tau)}=\frac{\left|\cA_v(\tau) \right|\epsilon^2(\tau)}{3\widetilde{K}\left(1-\sum_{j \in \cB_v(\tau)}p_{v,j}(\tau) \right)} \leq \frac{4}{9}\epsilon(\tau) =\frac{4}{9}\epsilon_{i}(\tau),
\end{equation*}
where in the last inequality, we bound $\left|\cA_v(\tau) \right|/\widetilde{K} \leq1$, use \eqref{boundpj} to bound $\left(1-\sum_{j \in \cB_v(\tau)}p_{v,j}(\tau) \right)$, and bound $\epsilon^2(\tau) < \epsilon(\tau)$. For $i \in \cB_v(\tau)$,
\begin{equation*} 
\frac{\log((8VK\log_4T)/\delta)}{N_{v,i}(\tau)}=\frac{\epsilon_{i}^2(\tau)}{3} <  \frac{\epsilon_{i}(\tau)}{3} < \frac{4}{9}\epsilon_{i}(\tau).
\end{equation*}

Then, using a union bound on similar proof of $\sum_{t \in \mathcal{T}(\tau)} -\hat{C}_{v,i}(t)/N_{v,i}(\tau)$, the following holds
\begin{equation} \label{concentrationCorr}
\mathbb{P}\left(\left| \frac{\sum_{t \in \mathcal{T}(\tau)} \hat{C}_{v,i}(t)}{N_{v,i}(\tau)} \right| > \frac{2 C(\tau)}{N(\tau)} +\frac{4}{9} \epsilon_{i}(\tau) \right) \leq \frac{\delta}{4VK\log_4T}  .
\end{equation}

By a union bound over \eqref{concentrationSto} and \eqref{concentrationCorr}, we get $| \hat{\mu}_{v,i}(\tau)-\mu_i    |\leq   \mu_i +2\epsilon_{i}(\tau) + 2 C(\tau)/N(\tau)$.

\end{proof}

\textbf{Lemma \ref{Tmpconcentration2All} (restated).} Let define event $\mathcal{E}$ as
\begin{equation*}
\mathcal{E}= \left\{\forall v,i,\tau: \left| \hat{\mu}_{i}(\tau)-\mu_i \right|   \leq  \mu_i +2\epsilon_{i}(\tau) + \frac{2 C(\tau)}{N(\tau)}, \tilde{N}_{v,i}(\tau)\leq 3N_{v,i}(\tau) \right\}.
\end{equation*}

Then, we hold that $\mathbb{P}\left[\mathcal{E}\right] \geq 1-\delta$.

\begin{proof}
We apply union bound over \eqref{concentration2Estima} and \eqref{concentration2Pull}such that
\begin{equation*}
\mathbb{P}\left(\left| \hat{\mu}_{v,i}(\tau)-\mu_i \right|   >  \mu_i +2\epsilon_{i}(\tau) + \frac{2 C(\tau)}{N(\tau)}\vee \tilde{N}_{v,i}(\tau) > 3N_{v,i}(\tau) \right) \leq \frac{\delta}{VK\log_4T} .
\end{equation*}
Again, by a union bound, we hold $\left|\hat{\mu}_{v,i}(\tau)-\mu_i \right|   >  \mu_i +2\epsilon_{i}(\tau) + \frac{2 C(\tau)}{N(\tau)}$ for all $v,i,\tau$ (at most $\log_4 T$ epochs).
Based on the definition of $\hat{\mu}_{i}(\tau)$, for all $v,i,\tau$, we have that
\begin{equation*} 
\begin{split}
\hat{\mu}_{i}(\tau) -\mu_i =\frac{1}{|\{v \in [V]:i \in \mathcal{K}_v(\tau)\}|} \sum_{v \in [V]:i \in \mathcal{K}_v(\tau)} \left( \hat{\mu}_{v,i}(\tau) -\mu_i \right) \leq   \mu_i+ 2\epsilon_{i}(\tau)+ \frac{C(\tau)}{N(\tau)} .
 \end{split}
\end{equation*}

Combing the above analysis, we get the desired result.
\end{proof}

\section{Algorithm Properties: proof of Lemma \ref{lemmaforStochastic} and Lemma \ref{property}}\label{algprop}

As the proof of Lemma \ref{lemmaforStochastic} will use the results of Lemma \ref{property}, we first prove Lemma \ref{property}. Here, we provide all useful properties of our algorithm, and these properties will be used for the following proof.

\textbf{Lemma \ref{property} (restated).}
The following holds.
\begin{enumerate}[(i)]
\item If $i \in \cA(\tau)$, then, $\epsilon(\tau)=\epsilon(d_i)/2$; \label{propi}
\item $7\epsilon(d_i) \geq \epsilon_{i}(d_i)$. \label{fact1}
\item If $i \in \cB(\tau)$, then, $\epsilon_i(\tau)=\epsilon(d_i)$. \label{propii}
\item If $i \in \cB(\tau)$, then, $\epsilon(d_i)=\epsilon_i(\tau) \geq 2\epsilon(\tau)$; \label{propiii}
\item $\epsilon(\tau) \leq \epsilon_i(\tau)$ for all arms $i \in [K]$; \label{propiv}
\end{enumerate}

\begin{proof}
Proof of \eqref{propi}: In our algorithm, $i \in \cA(\tau)$ implies that arm $i$ must hold \eqref{check1} at epoch $\tau-1$ and therefore the algorithm sets $d_i=\tau-1$ for epoch $\tau$. As $\epsilon(\tau)=\epsilon(\tau-1)/2$, the proof is evident.

Proof of \eqref{fact1}: For this proof, let consider two cases. In the first case, the arm $i$ does not hold \eqref{check1} for all epochs. We note that such a case occurs if and only if arm $i $ is deactivated in epoch $\tau=2$. Then, this case implies that $d_i=1$ (recall that we initialize $d_i=1$). Hence, we hold $\epsilon(d_i) = \epsilon_{i}(d_i)$ as we initialize $\epsilon(\tau=1)=\epsilon_{i}(\tau=1)$.

If arm $i$ is \textit{not} deactivated in epoch $\tau=2$, there \textit{must} exist an epoch such that arm $i$ holds \eqref{check1}. This is due to the reason that if arm $i$ is active in epoch $\tau=2$, then, it must suffice \eqref{check1} in epoch $\tau=1$. Thus, we have, at worst, $d_i=1$. In this case, we have
\begin{equation*} 
\begin{split}
14\epsilon(d_i)  \geq &\max\limits_{j \in \cA(\tau) \cup \mathcal{H}(\tau)} \left\{ \hat{\mu}_{j}(\tau) +2\epsilon_{j}(\tau) \right\} - \hat{\mu}_{i}(d_i) \\ 
 \geq &  \hat{\mu}_{i}(d_i)+2\epsilon_{i}(d_i)   - \hat{\mu}_{i}(d_i) \\
\geq & 2\epsilon_{i}(d_i) ,
\end{split}
\end{equation*}
which immediately leads to $7\epsilon(d_i) \geq \epsilon_{i}(d_i). $ Combing two cases, we get the desired result.

Proof of \eqref{propii}: This holds due to the construction of our algorithm (See Algorithm \ref{alg1}, line 21).

Proof of \eqref{propiii}: In our algorithm, $\cB(\tau)=\emptyset$ for $\tau=1$. As a result, $i \in \cB(\tau)$ only occurs for $\tau \geq 2$.
For $i \in \cB(\tau)$, arm $i$ does not hold \eqref{check1} at epoch $\tau-1$ and therefore the maximum possible value of $d_i$ is $\max\{1,\tau-2\}$. If the maximum possible value of $d_i$ is $1$, then, for $i \in \cB(\tau)$, $\epsilon(d_i=1) =\epsilon_i(\tau) \geq 2\epsilon(\tau)$. If the maximum possible value of $d_i$ is $\tau-2$, then, we hold $\epsilon(d_i)\geq \epsilon(\tau-2)=4\epsilon(\tau)$. From \eqref{propii}, we have $\epsilon_i(\tau)=\epsilon(d_i)\geq 4\epsilon(\tau)$, which completes the proof.

Proof of \eqref{propiv}: For this proof, we consider two cases. If $i \in \cA(\tau)$, this naturally holds $\epsilon_i(\tau)=\epsilon(\tau)$ due to the construction of our algorithm. Then, if $i \in \cB(\tau)$, we hold $\epsilon_i(\tau) \geq 2\epsilon(\tau) > \epsilon(\tau)$ due to \eqref{propiii}, which concludes the proof.

\end{proof}

\textbf{Lemma \ref{lemmaforStochastic} (restated).}
In the stochastic setting, with probability at least $1-\delta$, the following holds.
\begin{enumerate}[(i)]
\item The optimal arm $i^*$ is active for all $\tau$; \label{lemmai}
\item If $i \in \cA(\tau)$ and $i \in \cB(\tau+1)$, then, $\Delta_i> 8\epsilon(d_i)$; \label{lemmaii}
\item $ \mathcal{H}(\tau)=\emptyset$ for all $\tau$.  \label{lemmaiii}
\end{enumerate}

\begin{proof}
We prove this lemma by iteratively arguing \eqref{lemmai}, \eqref{lemmaii}, and \eqref{lemmaiii} for each epoch. 
Recall that the algorithm sets $\mathcal{H}(\tau)=\emptyset$, $\cB(\tau)=\emptyset$, and $\cA(\tau)=[K]$ for $\tau=1$. Then,
for any $i \neq i^*$ and $\tau=1$, the following holds
\begin{equation}\label{optimalactive}
\begin{split}
&\hat{\mu}^*(\tau)  - \hat{\mu}_{i^*} (\tau) \\
= &\max_{j \in \cA(\tau) \cup \mathcal{H}(\tau)} \left\{ \hat{\mu}_{j}(\tau) + 2\epsilon_j(\tau) \right\} - \hat{\mu}_{i^*} (\tau)  \\
= &\max_{j \in \cA(\tau)} \left\{ \hat{\mu}_{j}(\tau) + 2\epsilon_j(\tau) \right\} - \hat{\mu}_{i^*} (\tau)  \\
\leq   &\max_{j \in \cA(\tau) } \left\{ \mu_{j}+4\epsilon_j(\tau)  \right\}  - \left( \mu_{i^*} -2\epsilon(\tau)    \right) \\
=  &\max_{j \in \cA(\tau)} \left\{ \mu_{j} \right\}  -  \mu_{i^*} +6\epsilon(\tau)    \\
 \leq  & 6\epsilon(\tau). 
\end{split}
\end{equation}

The first inequality of \eqref{optimalactive} comes from Lemma \ref{Tmpconcentration2All} with $C(\tau)=0$. The last inequality of \eqref{optimalactive} follows $
\max_{j \in \cA(\tau) } \left\{ \mu_{j} \right\}\leq \max_{j \in [K]} \left\{ \mu_{j} \right\}=\mu_{i^*}$. It is evident that \eqref{optimalactive} contradicts to the deactivated condition, and it holds for all $\mathcal{H}(\tau) =\emptyset$.

We then show \eqref{lemmaii} holds in the stochastic setting. 
Suppose that $\tau+1$ is the \textit{first epoch} in which there exists an arm $ i \in [K]$ such that $i \in \cB(\tau+1)$. Such a case implies that arm $i$ contradicts to \eqref{check1} at epoch $\tau$, and thus we hold $i \in \cA(\tau)$, $i \in \cB(\tau+1)$. Since $\mathcal{H}(\tau)=\emptyset$ is empty \textit{before} epoch $\tau+1$, we have $i^* \in \cA(\tau)$ and $i^* \in \cA(\tau+1)$. Then, for epoch $\tau$ and the deactivated arm $i$, the following holds
\begin{equation} \label{deactivatei}
\begin{split}
&\hat{\mu}^*(\tau)  - \hat{\mu}_{i} (\tau) -14\epsilon(\tau)  \\
\geq  &\max\limits_{j \in \cA(\tau)} \left\{ \hat{\mu}_{j}(\tau) +2\epsilon_j(\tau) \right\} - \hat{\mu}_{i} (\tau)-14\epsilon(\tau)   \\
\geq &\max\limits_{j \in \cA(\tau) } \left\{ \mu_{j}  \right\} - \left( \mu_{i} +2\epsilon_i(\tau)    \right)  -14\epsilon(\tau) \\
= &\max\limits_{j \in \cA(\tau) } \{ \mu_{j}\}  - \mu_{i} -16\epsilon(\tau)\\
 \geq  &\Delta_i- 16\epsilon(\tau) ,
\end{split}
\end{equation}
where the equality of \eqref{deactivatei} is due to $\epsilon(\tau)=\epsilon_i(\tau)$ for $i \in \cA(\tau)$. Recall that when arm $i$ is deactivated at $\tau$, it should satisfy $\hat{\mu}^*(\tau)    - \hat{\mu}_{i} (\tau) -14\epsilon(\tau) >0$, which implies $\Delta_i> 16\epsilon(\tau) $. As $d_i=\tau-1$ is the last epoch that arm $i$ satisfies \eqref{check1}, we hold $\epsilon(d_i)=\epsilon(\tau-1)=2\epsilon(\tau)$.
Hence, we have $\Delta_i> 8\epsilon(d_i) $.

Finally, we show that in the stochastic setting, $\mathcal{H}(\tau)$ is an empty set for all epochs (i.e., the deactivated arms in $\cB(\tau+1)$ will not be reactivated).
Again, due to the fact that $i^* \in \cA(\tau+1)$, the following holds for any arm $i \in \cB(\tau+1)$.
\begin{equation} \label{reactivatei}
\begin{split}
&\max_{j \in \cA(\tau+1)} \hat{\mu}_{j}(\tau+1) - \hat{\mu}_{i}(\tau+1)   \\
\geq   &\max_{j \in \cA(\tau+1)}\{\mu_{j} \}-2\epsilon(\tau+1)   - \left( \mu_{i} + 2\epsilon_i(\tau+1)    \right) \\
\overset{(a)}{=} &\Delta_i- 2\epsilon(\tau+1) -2\epsilon(d_i)    \\
\overset{(b)}{\geq} &8\epsilon(d_i)  - 2\epsilon(\tau+1) -2\epsilon(d_i)    \\
\overset{(c)}{\geq} &4\epsilon(d_i)  . 
\end{split}
\end{equation}

The inequality (a) is due to $\epsilon_i(\tau+1)=\epsilon(d_i)$ for $i \in \cB(\tau+1)$ (see Lemma \ref{property}, \eqref{propii}). The inequality (b) follows the fact that $\Delta_i> 8\epsilon(d_i) $. The inequality (c) holds because $\epsilon(\tau+1) \leq \epsilon(d_i) $. As a consequence, \eqref{reactivatei} contradicts to the condition of reactivation, which further implies that $\mathcal{H}(\tau+1)=\emptyset$. Subsequently, $\cA(\tau+1) \cup \mathcal{H}(\tau+1)=\cA(\tau+1) $ holds, and this, again, leads to the argument that the arm $i^*$ will not be deactivated in epoch $\tau+2$. By iteratively repeating the analysis of \eqref{optimalactive}, \eqref{deactivatei} and \eqref{reactivatei} over all epochs, we get the desired results.
\end{proof}

\section{Regret and Communication: proof of Theorem \ref{theorem1}} \label{Appendixregret}

\subsection{Communication analysis}
Recall that the communication occurs in line \ref{algComm}, line \ref{algComm2}, and line \ref{algComm3}. Each of those messages, including $ \cA_v(\tau)$, $\cB_v(\tau)$, $  \{ \hat{\mu}_{v,i}(\tau)\}_{i \in \mathcal{K}_v(\tau) }$, and $  \{\epsilon_i(\tau)\}_{i \in \mathcal{K}_v(\tau) }$ has the size at most $\widetilde{K}=\lceil K/V \rceil+1$. Since we have that
$\widetilde{K}=\lceil K/V \rceil+1 <K/V +2 \leq 3K/V$ for $V \leq K$, the communication cost in an epoch across all agents is $O(K)$ bits. As the algorithm proceeds in at most $O(\log T)$ epochs and the communication occurs for each epoch, the total communication cost is $O(K\log T)$ bits.

\subsection{Technical lemmas for regret bound}
We define $R_{v,i}(\tau)=N_{v,i}(\tau) \Delta_{i}$, and $\eta(d_i)$ as
\begin{equation*}
\eta(d_i) =\sum_{s=1}^{d_i} \frac{C(s)}{2^{4d_i-4s}N(s)}.
\end{equation*}

\begin{lemma} \label{lemma2Ri}
If $\epsilon(d_i) \geq \Delta_i/\alpha $ for some constant $\alpha > 1$, then, the following holds,
\begin{equation*}
R_{v,i}(\tau) \leq   \left\{
\begin{aligned}
&\frac{3\alpha^2 \log((8K\log_4 T)/\delta) }{ \Delta_{i}},  &  \quad i \in  \cB_v(\tau), \\
&\frac{12\alpha^2 \widetilde{K}\log((8K\log_4 T)/\delta) }{ \left| \cA_v(\tau)\right| \Delta_{i}},  &  \quad i \in  \cA_v(\tau) .
\end{aligned}
\right.
\end{equation*}
\end{lemma}
\begin{proof}
Note that based on the construction of our algorithm, $i \in \cA_v(\tau)$ must imply that $i \in \cA(\tau)$. This argument is also applicable for $i \in \cB_v(\tau)$.
For $i \in \cA(\tau)$, we use \eqref{propi} of Lemma \ref{property} to get $\epsilon(\tau)=\epsilon(d_i)/2$. Thus, in the case of $i \in \cA(\tau)$, $R_{v,i}(\tau)$ is upper-bounded as
\begin{equation} \label{Case1RiA}
R_{v,i}(\tau)=N_{v,i}(\tau) \Delta_{i} \leq  \frac{3 \widetilde{K}\log((8K\log_4 T)/\delta) }{\left|\cA_v(\tau)  \right|\epsilon^2(\tau) }   \Delta_{i} \leq  \frac{ 12 \alpha^2  \widetilde{K}\log((8K\log_4 T)/\delta) }{\left|\cA_v(\tau)  \right| \Delta_{i}} .
\end{equation}

For $i \in \cB_v(\tau)$ (thereby $i \in \cB(\tau)$), we use \eqref{propii} of Lemma \ref{property} to get $\epsilon_i(\tau)=\epsilon(d_i)$, and then
\begin{equation} \label{Case1RiB}
R_{v,i}(\tau)=N_{v,i}(\tau) \Delta_{i} \leq  \frac{3\log((8K\log_4 T)/\delta) }{\epsilon_i^2(\tau) }   \Delta_{i} \leq  \frac{3\alpha^2 \log((8K\log_4 T)/\delta) }{ \Delta_{i}}  .
\end{equation}
\end{proof}

\textbf{Lemma \ref{lemma2Delta} (restated).}
If $\Delta_{i} > 32\eta(d_i)$, then, with probability at least $1-\delta$, $\epsilon(d_i) \geq \Delta_i/32$.

\begin{proof}
Since the following analysis is based on the fixed arm $i$, agent $v$ and epoch $\tau$, $d_i$ is the last epoch up to epoch $\tau$ such that arm $i$ holds \eqref{check1} unless \textit{an exception} that $d_i=1$. Recall that the algorithm initializes $d_i=1$ and $i \in \cA(\tau=1)$ for all arms; then it is entirely possible that arm $i$ may not hold \eqref{check1} for $d_i=\tau=1$.
Here, we first show that when $d_i=1$, the desired result $\epsilon(d_i=1) =1/14 \geq \Delta_i/32$ trivially holds because $\Delta_i \leq 1$ implies that $\Delta_i/32 \leq 1/32 < 1/14$. Then, one can bound $\epsilon(d_i)$ by considering the following three cases for $d_i \geq 2$, and thus arm $i$ must hold \eqref{check1} in epoch $d_i $.

\textbf{Case 1: $i^* \in \cA(d_i) \cup \mathcal{H}(d_i)$.} As $d_i$ is the last epoch up to epoch $\tau$ such that arm $i$ holds \eqref{check1} (see line \ref{algDeact} of Algorithm \ref{alg1}), the following holds
\begin{equation}  \label{similar3}
\begin{split}
14\epsilon(d_i)  \geq &  \max\limits_{j \in \cA(d_i) \cup \mathcal{H}(d_i)} \left\{ \hat{\mu}_{j}(d_i) +2\epsilon_{j}(d_i) \right\} - \hat{\mu}_{i}(d_i)  \\ 
\overset{(a)}{\geq} &  \hat{\mu}_{i^*}(d_i)+2\epsilon_{i^*}(d_i)   - \hat{\mu}_{i}(d_i) \\
\overset{(b)}{\geq} & \Delta_i - \frac{2C(d_i)}{N(d_i)}- 2\epsilon_{i}(d_i)- \frac{2C(d_i)}{N(d_i)} \\
\overset{(c)}{\geq}  & \Delta_i - \frac{4C(d_i)}{N(d_i)}- 14\epsilon(d_i) \\
\overset{(d)}{\geq}  & \Delta_i - \frac{\Delta_i}{8}- 14\epsilon(d_i) ,
\end{split}
\end{equation}
where:
\begin{itemize}
\item The inequality (a) is due to for $i^* \in \cA(d_i) \cup \mathcal{H}(d_i)$, we have that
\begin{equation*}
 \max\limits_{j \in \cA(\tau) \cup \mathcal{H}(\tau)} \left\{ \hat{\mu}_{j}(\tau) +2\epsilon_{j}(\tau) \right\}  \geq \hat{\mu}_{i^*}(d_i)+2\epsilon_{i^*}(d_i) . \end{equation*}
\item The inequality (b) follows Lemma \ref{Tmpconcentration2All}. 
\item The inequality (c) comes from \eqref{fact1} of Lemma \ref{property}. 
\item The inequality (d) holds because
\begin{equation} \label{etabound}
\Delta_i > 32\eta(d_i)  = \sum_{s=1}^{d_i} \frac{32C(s)}{2^{4d_i-4s}N(s)} \geq  \frac{32C(d_i)}{N(d_i)}.
\end{equation}
\end{itemize}

From \eqref{similar3}, we have $\epsilon(d_i) \geq \Delta_i/32$.

\textbf{Case 2: $8\epsilon_{i^*}(d_i) \leq \Delta_i$ and $i^* \in \cB(d_i) \backslash \mathcal{H}(d_i)$.} 
For $i^* \in \cB(d_i) \backslash \mathcal{H}(d_i)$, we have
$
\max\limits_{j \in \cA(d_i)} \hat{\mu}_{j}(d_i) \geq 4\epsilon(d_{i^*}) +\hat{\mu}_{i^*}(d_i)
$ as $i^*$ contradicts the reactivation condition (see line \ref{algReact} of Algorithm \ref{alg1}). Further, we have
\begin{equation} \label{fact2}
 \max\limits_{j \in \cA(d_i) \cup \mathcal{H}(d_i)} \{ \hat{\mu}_{j}(d_i)+2\epsilon_{j}(d_i)  \} >  \max\limits_{j \in \cA(d_i) } \hat{\mu}_{j}(d_i)  \geq  4\epsilon(d_{i^*}) +\hat{\mu}_{i^*}(d_i) =4\epsilon_{i^*}(d_i) +\hat{\mu}_{i^*}(d_i) .
\end{equation}

Note that the last inequality in \eqref{fact2} comes from \eqref{propii} of Lemma \ref{property} (as $i^* \in \cB(d_i) \backslash \mathcal{H}(d_i)$). Again, as arm $i$ holds \eqref{check1}, we have
\begin{equation} \label{inotinA}
\begin{split}
14\epsilon(d_i)  \geq &  \max\limits_{j \in \cA(d_i) \cup \mathcal{H}(d_i)} \{ \hat{\mu}_{j}(d_i)+2\epsilon_{j}(d_i)  \} - \hat{\mu}_{i}(d_i) \\
> &  \hat{\mu}_{i^*}(d_i) +4\epsilon_{i^*}(d_i)- \hat{\mu}_{i}(d_i) \\
\geq &\mu_*+2\epsilon_{i^*}(d_i) -\frac{2C(d_i)}{N(d_i)}- \mu_i - 14\epsilon(d_i)- \frac{2C(d_i)}{N(d_i)} \\
\geq &\Delta_i - 14\epsilon(d_i)-\frac{\Delta_i}{8},
\end{split}
\end{equation}
where the second inequality uses \eqref{fact2}; the third inequality follows Lemma \ref{Tmpconcentration2All} and \eqref{fact1} of Lemma \ref{property}; the last inequality uses $\epsilon_{i^*}(d_i) \geq 0$ and \eqref{etabound}. Then, we get $\epsilon(d_i) \geq \Delta_i /32$

\textbf{Case 3: $8\epsilon_{i^*}(d_i) > \Delta_i$ and $i^* \in \cB(d_i) \backslash \mathcal{H}(d_i)$.} 
As $i^* \in \cB(d_i) \backslash \mathcal{H}(d_i)$, $i^*$ contradicts the reactivation condition, which implies that
\begin{equation*} 
\begin{split}
4\epsilon(d_{i^*}) \leq & \max_{j \in \cA(d_i)} \hat{\mu}_{j}(d_i) - \hat{\mu}_{i^*}(d_i)     \\
\overset{(a)}{\leq} & \max_{j \in \cA(d_i)}\left\{ \mu_j +2\epsilon_j(d_i)\right\}+ \frac{2C(d_i) }{N(d_i)} -\mu_{i^*}+2\epsilon_{i^*}(d_i)+ \frac{2C(d_i)}{N(d_i)}\\
\overset{(b)}{\leq}  &  \frac{4C(d_i) }{N(d_i)}+2\epsilon(d_i) +2\epsilon_{i^*}(d_i)\\
\overset{(c)}{\leq}  &  \frac{4C(d_i) }{N(d_i)}+3\epsilon_{i^*}(d_i) \\
\overset{(d)}{\leq}  & \frac{\Delta_i}{8}+3\epsilon_{i^*}(d_i),
\end{split}
\end{equation*}
\begin{itemize}
\item The inequality (a) follows Lemma \ref{Tmpconcentration2All}. 
\item The inequality (b) is due to 
$
\max_{j \in \cA(d_i)}\left\{ \mu_j \right\}\leq \max_{j \in [K]}\left\{ \mu_j \right\} =\mu_{i^*}.
$
\item The inequality (c) follows \eqref{propiii} of Lemma \ref{property} that $ \epsilon(d_i) \leq \epsilon_{i^*}(d_i)/2 $ for $i^* \in \cB(d_i) \backslash \mathcal{H}(d_i)$. 
\item The inequality (d) holds because of \eqref{etabound}.
\end{itemize}

Note that as $i^* \in \cB(d_i)$, we have $\epsilon(d_{i^*}) =\epsilon_{i^*}(d_{i})$ (Recall \eqref{propii} of Lemma \ref{property}).
Thus, we get $8\epsilon_{i^*}(d_i) \leq \Delta_i$, which contradicts to $8\epsilon_{i^*}(d_i) > \Delta_i$.

\end{proof}

\subsection{Proof of regret bound}\label{summationformReg}
By Lemma \ref{Tmpconcentration2All}, the regret, with probability at least $1-\delta$, is upper-bounded as
\begin{equation*}
R_T  =   \sum_{\tau=1}^L  \sum_{v=1}^V \sum_{i \neq i^* , i \in \mathcal{K}_v(\tau)} \tilde{N}_{v,i}(\tau) \Delta_{i} \\
\leq   \sum_{\tau=1}^L \sum_{v=1}^V   \sum_{i \neq i^* , i \in \mathcal{K}_v(\tau) } 3N_{v,i}(\tau)  \Delta_{i}  
=  3R(T,\cA)  + 3R(T,\cB) ,
\end{equation*}
where
\begin{equation*}
R(T,\cA)= \sum_{\tau=1}^L \sum_{v=1}^V\sum_{i \in \cA_v(\tau)  , i \neq i^*}  R_{v,i}(\tau), \ \text{and} \
R(T,\cB)= \sum_{\tau=1}^L \sum_{v=1}^V \sum_{i \in \cB_v(\tau), i \neq i^*}  R_{v,i}(\tau)  .
\end{equation*}

Note that we slightly abuse the notation of $L$, and $L$ here is the maximum epoch that the algorithm proceeds up to $T$.

\subsection*{Bounding $R(T,\cA)$}\label{boundRTA}
\textbf{Case 1:} $\Delta_{i} \leq 32 \eta(d_i)$. We use Lemma \ref{property} \eqref{propi} to get $\epsilon(d_i)=2\epsilon(\tau)$ and then $R_{v,i}(\tau)$ in this case is upper-bounded as,
\begin{equation*}
\begin{split}
R_{v,i}(\tau) \leq  &  \frac{96\widetilde{K}\log((8K\log_4 T)/\delta) }{\left|\cA_v(\tau)  \right|\epsilon^2(\tau) } \eta(d_i) \\
= &\frac{384\widetilde{K}\log((8K\log_4 T)/\delta) }{\left|\cA_v(\tau)  \right| \epsilon^2(d_i) }  \sum_{s=1}^{d_i} \frac{ C(s)}{  2^{4d_i-4s} \left(\frac{3\widetilde{K}\log((8K \log_4 T)/\delta)}{\epsilon^2(s)} \right)}  \\
= &\frac{128}{\left|\cA_v(\tau)  \right|}  \sum_{s=1}^{d_i} \frac{C(s)2^{2d_i-2s} }{  2^{4d_i-4s}}  .
\end{split}
\end{equation*}

Then, the regret under such a case is upper-bounded as,
\begin{equation*}
\begin{split}
 R(T,\cA) \leq & 128  \sum_{\tau=1}^L    \sum_{v=1}^V \sum_{i \in \cA_v(\tau) , i \neq i^*}\frac{1 }{\left|\cA_v(\tau)  \right|} \sum_{s=1}^{d_i} \frac{  C(s)2^{2d_i-2s} }{  2^{4d_i-4s}} \\
\leq &128 \sum_{s=1}^L \sum_{v=1}^V  C(s) \sum_{i \in \cA_v(s) , i \neq i^*}\frac{1 }{\left|\cA_v(s)  \right|} \sum_{d_i=s}^{L}   \frac{2^{2d_i-2s} }{  2^{4d_i-4s}} \\
\leq & 128 \sum_{s=1}^L \sum_{v=1}^V   C(s) \sum_{i \in \cA_v(s) , i \neq i^*}\frac{1 }{\left|\cA_v(s)  \right|} \sum_{q=1}^{\infty} 4^{-q}\\
\leq &128  \sum_{s=1}^L  \sum_{v=1}^V  C(s)\\
\leq&128 VC .
\end{split}
\end{equation*}

\textbf{Case 2: $\Delta_{i} > 32 \eta(d_i)$.} As $V \leq K$, we have that $\widetilde{K}=\lceil K/V \rceil+1 <K/V +2 \leq 3K/V$. Using Lemma \ref{lemma2Ri} and Lemma \ref{lemma2Delta}, we upper-bound $R(T,\cA)$ as
\begin{equation*}
R(T,\cA) \leq \sum_{\tau=1}^L\sum_{v=1}^V  \sum_{i \in \cA_v(\tau) , i \neq i^*}  \frac{12\cdot 32^2 \widetilde{K} \log((8VK\log_4 T)/\delta) }{ \left| \cA_v(\tau)\right| \Delta_{i}}.
\end{equation*}

%\begin{equation*}
%8\epsilon(s) < \Delta_{i} \leq \frac{1}{\Delta_{\max}} \sum_{q=1}^{s} \frac{ C_{i,j}(q)}{  2^{s-q} N(q)}  .
%\end{equation*}

\subsection*{Bounding $ R(T,\cB)$}
\textbf{Case 1:} $\Delta_{i} \leq 32 \eta(d_i)$. For $i \in \cB(\tau)$, we use Lemma \ref{property} \eqref{propii} to get $\epsilon(d_i)=\epsilon_i(\tau)$. Then, $R_{v,i}(\tau)$ is upper-bounded as,
\begin{equation*}
\begin{split}
R_{v,i}(\tau) \leq  \frac{96\log((8K\log_4 T)/\delta) }{\epsilon_i^2(\tau) } \eta(d_i) \leq \frac{32}{\widetilde{K}}  \sum_{s=1}^{d_i} \frac{ C(s)2^{2d_i-2s} }{  2^{4d_i-4s}}  .
\end{split}
\end{equation*}
Then, $R(T,\cB)$ in this case is upper-bounded as,
\begin{equation*}
\begin{split}
 R(T,\cB) \leq & \frac{32 }{\widetilde{K}} \sum_{\tau=1}^L\sum_{v=1}^V  \sum_{i \in \cB_v(\tau) , i \neq i^*}  \sum_{s=1}^{d_i} \frac{ C(s)2^{2\tau-2s} }{  2^{4\tau-4s}} 
\leq 32 \sum_{s=1}^L \sum_{v=1}^V   C(s) \sum_{q=1}^{\infty} 4^{-q} \leq 32 VC.
\end{split}
\end{equation*}

\textbf{Case 2: $\Delta_{i} > 32 \eta(d_i)$. }
By Lemma \ref{lemma2Ri} and Lemma \ref{lemma2Delta}, $R(T,\cB)$ is upper-bounded as
\begin{equation*}
 R(T,\cB)  \leq  \sum_{\tau=1}^L  \sum_{v=1}^V  \sum_{i \in \cB_v(\tau) , i \neq i^*}   \frac{3\cdot 32^2 \log((8K\log_4 T)/\delta) }{\Delta_{i}} .
\end{equation*}

Combing the $ R(T,\cA)$ and $ R(T,\cB)$, we get the regret $R_T$ as
\begin{equation*}
 O\left(VC+ \log((VK\log T)/\delta)  \sum_{\tau=1}^L \left( \sum_{v=1}^V\frac{ \widetilde{K} }{  \left| \cA_v(\tau)\right| }  \sum_{i \in \cA_v(\tau) , i \neq i^*}   \frac{ 1 }{\Delta_{i}} + \sum_{v=1}^V \sum_{i \in \cB_v(\tau) , i \neq i^*}   \frac{  1}{ \Delta_{i}}  \right)   \right).
\end{equation*}

\subsection{Concise regret bound}
We here give a concise regret bound by lower-bounding $\Delta_i \geq \Delta_{\min}$. For 
$\Delta_{i} \leq 32 \eta(d_i)$, both $ R(T,\cA)$ and $ R(T,\cB)$ are bounded by $O(VC)$. As a consequence, we only need to refine $ R(T,\cA)$ and $ R(T,\cB)$ for the case $\Delta_{i} > 32 \eta(d_i)$.

Let first refine $ R(T,\cA)$ for $\Delta_{i} > 32 \eta(d_i)$. As $V \leq K$, we have $\widetilde{K}=\lceil K/V \rceil+1 <K/V +2 \leq 3K/V$, we upper-bound $R(T,\cA)$ as
\begin{equation*}
R(T,\cA) \leq   
\sum_{\tau=1}^L \sum_{v=1}^V \sum_{i \in \cA_v(\tau) , i \neq i^*}   \frac{36\cdot 32^2 K\log((8K\log_4 T)/\delta) }{V \left| \cA_v(\tau)\right| \Delta_{i}}  .
\end{equation*}

Similarly, by $\Delta_i \geq \Delta_{\min}$, $ R(T,\cA)$ is upper-bounded as
\begin{equation*}
R(T,\cA) \leq    \frac{24\cdot 32^2 K \log_4 T \log((8K\log_4 T)/\delta) }{ \Delta_{\min}}.
\end{equation*}

Then, we refine $ R(T,\cB)$ as
\begin{equation*}
\begin{split}
 R(T,\cB)  \leq  & \sum_{\tau=1}^L  \sum_{v=1}^V  \sum_{i \in \cB_v(\tau) , i \neq i^*}   \frac{3\cdot 32^2 \log((8K\log_4 T)/\delta) }{\Delta_{\min}}\\
 \leq  & \sum_{\tau=1}^L  \sum_{v=1}^V  \frac{3\cdot 32^2 \widetilde{K} \log((8K\log_4 T)/\delta) }{\Delta_{\min}}  \\
= &  \frac{9\cdot 32^2 K\log_4 T \log((8K\log_4 T)/\delta) }{\Delta_{\min}},
\end{split}
\end{equation*}
where the second inequality is due to  $\sum_{i \in \cB_v(\tau) , i \neq i^*} 1  < |\mathcal{K}_v|=\widetilde{K}$, and the last inequality, again, uses the fact that $\widetilde{K}=\lceil K/V \rceil+1 <K/V +2 \leq 3K/V$ for $V \leq K$.

\subsection{Expected regret}
For the proof of expected regret, we set $\delta = 1/T$, and then, one can trivially bound the regret by $O(T)$ with the failure probability $1/T$. Thus, this part yields the expected regret as $O(1)$, and the expected regret from the remaining probability at least $1-1/T$ can be bounded by using the result of $R_T$. Thus, the expected
regret is bounded by $
 O(V\mathbb{E}[C]+\frac{ K \log T \log(VT) }{  \Delta_{\min}}  ).
$

\section{Discussion} 
\subsection{Proof of Theorem \ref{ConnectRegret}}

\textbf{Theorem \ref{ConnectRegret} (restated).}
For any algorithm of $V \geq 1$ agents, with probability at least $1-1/VT$, we have that
$
R_T  =O(\overline{R}_T+\mathbb{E} \left[ C\right] + V \log(VT) ).$
Then, $\mathbb{E}[R_T] = \Theta(\overline{R}_T  +\mathbb{E} \left[ C\right] )$ and $\mathbb{E}[R_T] = \Theta(   \mathbb{E}[R'_T]+\mathbb{E} \left[ C\right] )$ hold for the two-armed bandit instance.

\begin{proof}
We first show $
R_T  =O(\overline{R}_T+\mathbb{E} \left[ C\right] + V \log(VT) )$. Recall the definition of $R_T$, and it can be rewritten as
\begin{equation*} 
\begin{split}
R_T  =   \sum_{t=1}^T \sum_{v=1}^V  \left( \mu_{i^*}-\mu_{i_v(t)}\right) = \sum_{t=1}^T \sum_{v=1}^V  \Delta_i \mathbb{I}\left\{i_v(t)=i \right\}.
\end{split}
\end{equation*}

By constructing a martingale $R_T- \mathbb{E}[R_T]$, $D_v(t)=
\left\{  \Delta_{i} (\mathbb{I}\{i_v(t)=i\}- \mathbb{P}\left( i_v(t)=i \right))\right\}_{t=0}^{\infty}
$ is the martingale difference sequence. Applying Freedman's inequality, the following holds with probability at least $1-1/VT$,
\begin{equation*}
\begin{split}
R_T - \mathbb{E}[R_T] \leq &\sum_{v=1}^V \left( \log(VT) + \sum_{t=1}^T \mathbb{E}\left[D_v(t)^2 \mid \mathcal{F}_{t-1} \right]  \right)\\
 \leq &V \log(VT) +\sum_{v=1}^V \sum_{t=1}^T    \Delta_{i} \mathrm{Var} \left( \mathbb{I}\{i_v(t)=i\}\right)\\
 \leq &V \log(VT) + \sum_{v=1}^V  \sum_{t=1}^T \Delta_{i} \mathbb{P}\left( i_v(t)=i \right) \\
 \leq &V \log(VT) +\mathbb{E}[R_T],
\end{split}
\end{equation*}

After some simple algebra, the following holds
\begin{equation}\label{reg2expectedreg}
\begin{split}
R_T  \leq V \log(VT) +2\mathbb{E}[R_T].
\end{split}
\end{equation}

Recall $\overline{R}_T$ in \cite{JMLRTsallis} is given as follows. (here we write it in a generic $V$-agent form, and if we set $V=1$, $\overline{R}_T$ is the same the one in \cite{JMLRTsallis}.)
\begin{equation*}
\overline{R}_T =  \max_{i } \mathbb{E} \left[ \sum_{t=1}^T  \sum_{v=1}^V \left( r_{v,i}(t)- r_{i_v(t)}(t)  \right) \right]. 
\end{equation*}

The regret $\overline{R}_T$ is further lower-bounded as
\begin{equation}\label{selfbounding}
\begin{split}
\overline{R}_T = & \max_{i } \mathbb{E} \left[ \sum_{t=1}^T  \sum_{v=1}^V \left( r_{v,i}(t)- r_{i_v(t)}(t)  \right) \right] \\
\geq  & \mathbb{E} \left[ \sum_{t=1}^T  \sum_{v=1}^V \left( r_{v,i^*}(t)- r_{i_v(t)}(t)  \right) \right] \\
\geq  & \mathbb{E} \left[ \sum_{t=1}^T  \sum_{v=1}^V \left( r^S_{v,i^*}(t)- r^S_{i_v(t)}(t)  \right) \right]
+\mathbb{E} \left[ \sum_{t=1}^T  \sum_{v=1}^V \left(r^S_{i_v(t)}(t)- r_{i_v(t)}(t)  \right) \right]\\
&+\mathbb{E} \left[ \sum_{t=1}^T  \sum_{v=1}^V \left( r_{v,i^*}(t)- r^S_{v,i^*}(t)  \right) \right]\\
\geq  & \mathbb{E} \left[ \sum_{t=1}^T  \sum_{v=1}^V \left( r^S_{v,i^*}(t)- r^S_{i_v(t)}(t)  \right) \right]
-2\mathbb{E} \left[ C\right] \\
=  & \mathbb{E} \left[ R_T  \right]
-2\mathbb{E} \left[ C\right] ,
\end{split}
\end{equation}
which immediately leads to $\mathbb{E}[R_T] \leq \overline{R}_T  +2 \mathbb{E} \left[ C\right] $. Combing the above results and \eqref{reg2expectedreg}, then,
\begin{equation*} 
R_T  \leq   2\left( \overline{R}_T+2\mathbb{E} \left[ C\right]  \right) + V \log(VT) .
\end{equation*}

Recall Theorem \ref{theorem2} that $ \mathbb{E}[R_T]=
\Omega (  \mathbb{E}[C] + \mathbb{E}[R'_T])
$ in the two-armed instance. Since $\overline{R}_T \leq \mathbb{E}[R'_T]$ and $\overline{R}_T =\mathbb{E}[R'_T]$ iff adversary is oblivious, i.e., the adversary fixes his/her strategy for all rounds, and the strategy is independent of agents' choices \cite{BanditBook1}, we have that $ \mathbb{E}[R_T]=
\Omega (  \mathbb{E}[C] + \overline{R}_T)
$. Thus, from \eqref{selfbounding} and $ \mathbb{E}[R_T]=
\Omega (  \mathbb{E}[C] + \overline{R}_T)
$, $ \mathbb{E}[R_T]=
\Theta (  \mathbb{E}[C] + \overline{R}_T)
$ holds.

From \eqref{selfbounding} and the fact $\overline{R}_T \leq \mathbb{E}[R'_T]$, it can be seen have that $\mathbb{E}[R_T] \leq \mathbb{E}[R'_T]  +2 \mathbb{E} \left[ C\right] $.
Again, using the result $ \mathbb{E}[R_T]=
\Omega (  \mathbb{E}[C] + \mathbb{E}[R'_T])
$ in Theorem \ref{theorem2}, we conclude that $ \mathbb{E}[R_T]=
\Theta (  \mathbb{E}[C] +  \mathbb{E}[R'_T])
$.

\end{proof}

\subsection{Discussion of regret in \cite{JMLRTsallis} and \cite{ImprovedTsallis}}  \label{discussofTsallis}
In \cite{JMLRTsallis} and \cite{ImprovedTsallis}, authors give the upper bound of $\mathbb{E}[R'_T]$. We here convert $\mathbb{E}[R'_T]$ to $\mathbb{E}[R_T]$.
From Corollary 8 in \cite{JMLRTsallis} and \eqref{selfbounding} in Theorem \ref{ConnectRegret}, $\mathbb{E}[R_T]$ with $V=1$ is bounded by
\begin{equation*} 
\begin{split}
\mathbb{E}[R_T]  \leq 2\left( \sum_{i\neq i^*} \frac{\log T}{\Delta_i}  +\sqrt{ \sum_{i\neq i^*} \frac{\log T}{\Delta_i} \mathbb{E} \left[ C\right]}  +2\mathbb{E} \left[ C\right]  \right)  .
\end{split}
\end{equation*}

Note that the expectation over $C$ cannot be dropped, as the corruption $C$ is a random variable that depends on
the randomization of stochastic rewards and the choices of agents.
By a similar method, the $\mathbb{E}[R_T]$ of \cite{ImprovedTsallis} is:
\begin{equation*} 
\begin{split}
\mathbb{E}[R_T]  \leq \left( \sum_{i\neq i^*} \frac{1}{\Delta_i} \log \left( \frac{KT}{ \sum_{i\neq i^*} \frac{1}{\Delta_i}}   \right) +\sqrt{ \sum_{i\neq i^*} \frac{1}{\Delta_i} \log \left( \frac{KT}{ \sum_{i\neq i^*} \frac{1}{\Delta_i}}   \right)  \mathbb{E} \left[ C\right]}  +2\mathbb{E} \left[ C\right]  \right).
\end{split}
\end{equation*}

As mentioned in \cite{ImprovedTsallis}, for $C=\Theta(\frac{TK}{\sum_{i\neq i^*} \frac{\log T}{\Delta_i}})$, adversarial pseudo-regret $ \overline{R}_T$ is improved by a multiplicative factor of $\sqrt{\frac{\log T}{\log \log T}}$. However, this implication does not hold for $\mathbb{E}[R_T]$ due to the additive term of corruption is unavoidable (see lower bound in Theorem \ref{theorem2}). When $C=\Theta(\frac{TK}{\sum_{i\neq i^*} \frac{\log T}{\Delta_i}})$, the expected regret grows linearly.

\end{document}